\documentclass{article} 
\usepackage{iclr2026_conference,times}

\usepackage{amsmath,amsfonts,bm}

\def\eqref#1{equation~\ref{#1}}

\def\1{\bm{1}}

\DeclareMathAlphabet{\mathsfit}{\encodingdefault}{\sfdefault}{m}{sl}
\SetMathAlphabet{\mathsfit}{bold}{\encodingdefault}{\sfdefault}{bx}{n}

\usepackage{url}
\usepackage{booktabs}       
\usepackage{amsfonts}       
\usepackage{nicefrac}       
\usepackage{microtype}      
\usepackage{xcolor}         
\usepackage{tabularx}
\usepackage{multirow}
\usepackage{amsmath,amssymb,amsthm}
\usepackage{graphicx}
\usepackage{wrapfig}
\usepackage{enumitem}
\usepackage{mathtools}
\usepackage{pifont}
\usepackage{algorithm}
\usepackage{algorithmic}
\usepackage{bbm}
\usepackage[T1]{fontenc}
\usepackage{hyperref}

\theoremstyle{definition}

\theoremstyle{plain}
\newtheorem{theorem}{Theorem}
\newtheorem{proposition}{Proposition}
\newtheorem{lemma}{Lemma}
\newtheorem{assumption}{Assumption}

\title{Accelerating Benchmarking of Functional Connectivity Modeling via Structure-aware Core-set Selection}

\author{
  \textbf{Ling Zhan}$^{1,2}$, \quad \textbf{Zhen Li}$^{1}$, \quad \textbf{Junjie Huang}$^{1}$, \quad \textbf{Tao Jia}$^{1,2,3}$\thanks{Corresponding author.} \\[0.5ex]
  \small $^1$College of Computer and Information Science, Southwest University, Chongqing, China \\
  \small $^2$Chongqing Key Laboratory of Brain-Inspired Cognitive Computing and \\[-0.2ex]
  \small Educational Rehabilitation for Children with Special Needs, Chongqing Normal University, China \\
  \small $^3$College of Computer and Information Science, Chongqing Normal University, China \\
  \small \texttt{\{zl0327, cs2003lz\}@email.swu.edu.cn, \{junjiehuang, tjia\}@swu.edu.cn}
}

\iclrfinalcopy 
\begin{document}

\maketitle

\begin{abstract}
Benchmarking the hundreds of functional connectivity (FC) modeling methods on large-scale fMRI datasets is critical for reproducible neuroscience. However, the combinatorial explosion of model--data pairings makes exhaustive evaluation computationally prohibitive, preventing such assessments from becoming a routine pre-analysis step. To break this bottleneck, we reframe the challenge of FC benchmarking by selecting a small, representative \emph{core-set} whose sole purpose is to preserve the relative performance ranking of FC operators. We formalize this as a ranking-preserving subset selection problem and propose \textbf{S}tructure-aware \textbf{C}ontrastive \textbf{L}earning for \textbf{C}ore-set \textbf{S}election (\textbf{SCLCS}), a self-supervised framework to select these core-sets. \textbf{SCLCS} first uses an adaptive Transformer to learn each sample's unique FC structure. It then introduces a novel \textbf{S}tructural \textbf{P}erturbation \textbf{S}core (\textbf{SPS}) to quantify the stability of these learned structures during training, identifying samples that represent foundational connectivity archetypes. Finally, while \textbf{SCLCS} identifies stable samples via a top-$k$ ranking, we further introduce a \textbf{density-balanced sampling strategy} as a necessary correction to promote diversity, ensuring the final core-set is both structurally robust and distributionally representative. On the large-scale REST-meta-MDD dataset, \textbf{SCLCS} preserves the ground-truth model ranking with just $10\%$ of the data, outperforming state-of-the-art (SOTA) core-set selection methods by up to $23.2\%$ in ranking consistency (nDCG@k). To our knowledge, this is the first work to formalize core-set selection for FC operator benchmarking, thereby making large-scale operators comparisons a feasible and integral part of computational neuroscience. Code is publicly available on \url{https://github.com/lzhan94swu/SCLCS}
\end{abstract}

\section{Introduction}
Methodological choices can substantially affect scientific reproducibility, as reflected in highly variable outcomes obtained from the same dataset, making systematic benchmarking increasingly important~\citep{kohli2024towards,qiu2024tfb,marek2022reproducible}.
This issue is especially acute in functional connectivity (FC) modeling, where hundreds of candidate statistical pairwise interactions (SPIs) require careful evaluation to ensure reliable conclusions~\citep{liu2025benchmarking,roell2025measure}.
Yet the computational cost of exhaustive evaluation makes it impractical to run as a routine pre-analysis step for data-driven model selection~\citep{ying2024automating,zhou2021review} (see the complexity analysis in \textbf{Appendix~\ref{app:comps}}).
To address this bottleneck, we propose a two-stage workflow: we first benchmark all candidate SPIs on a small, representative core-set to identify top performers, and then evaluate the selected SPI(s) on the full dataset for downstream analysis.
This workflow hinges on selecting a core-set that preserves the relative performance ranking of SPIs.

While core-set selection is well studied, most existing methods target a different goal: constructing a small training proxy for a single predictive model~\citep{feldman2020core,lee2024coreset,hong2024evolution}.
In our setting (Figure~\ref{problem}), the core-set must preserve the relative performance ranking across hundreds of candidate SPIs~\citep{liu2025benchmarking,cliff2023unifying}.
This ranking-preservation objective raises three challenges: (1) Formulating a selection criterion that targets cross-SPI ranking stability rather than single-model training loss. (2) Defining a principled, structure-aware notion of sample importance based on FC patterns (the targets of SPIs). (3) Reducing the brittleness of score-based top-$k$ selection, which can fail to generalize across sampling ratios and distort rankings.

In this work, we cast core-set selection for FC benchmarking as a ranking-preserving subset selection problem.
Rather than training a predictive model, we seek a subset that preserves the SPI ordering from the full dataset (Figure~\ref{problem}).
We evaluate on the REST-meta-MDD dataset~\citep{yan2019reduced,long2020altered}, a large multi-site resting-state fMRI dataset for MDD, which captures heterogeneity across acquisition sites and a large cohort.
We instantiate benchmarking with two tasks, brain fingerprinting~\citep{van2021makes} and MDD diagnosis~\citep{gallo2023functional}, both widely used in FC research~\citep{lu2024brain,otte2016major}.
For each task, we score each SPI by how well the resulting FC matrices separate within-class from between-class pairs using Spearman’s rank correlation~\citep{sedgwick2014spearman}, yielding an SPI ranking.
Core-set quality is measured by nDCG@k~\citep{wang2013theoretical} between the SPI rankings induced by the core-set and the full dataset.

We use SPIs as a validation case because benchmarking FC operators has been formalized as a well-defined task in recent work~\citep{liu2025benchmarking,cliff2023unifying,honari2021evaluating}. Based on this formulation, we propose \textbf{S}tructure-aware \textbf{C}ontrastive \textbf{L}earning for \textbf{C}ore-set \textbf{S}election (\textbf{SCLCS}).
As shown in Figure~\ref{problem}, \textbf{SCLCS} is built around a Transformer-based encoder that encodes sample-specific synchronization structure via an adaptively weighted fusion of attention heads. Under the assumptions of \textbf{Theorem~\ref{thm:universal}}, we show this encoder has universal approximation capacity for continuous SPI mappings.
We then define a \textbf{S}tructure \textbf{P}erturbation \textbf{S}core (\textbf{SPS}) to quantify the stability of these structures, and prioritize low-\textbf{SPS} samples to form a robust core-set.
Because naïve top-$k$ selection can be brittle for certain task structures, \textbf{SCLCS} augments it with a density-aware sampling strategy to improve diversity.
\textbf{SCLCS} learns in an identity-supervised contrastive manner, using subject identities to encourage stable “brain fingerprints”\citep{van2021makes} that SPI-based analyses aim to capture\citep{liu2025benchmarking,luppi2024systematic}. This yields task-agnostic representations suitable for benchmarking.
Finally, \textbf{SCLCS} is a pre-analysis acceleration tool that makes large-scale benchmarking computationally feasible, rather than a method for the final neuroscientific discovery task.

\begin{figure}[!tb]
  \centering
  \includegraphics[width=0.9\linewidth]{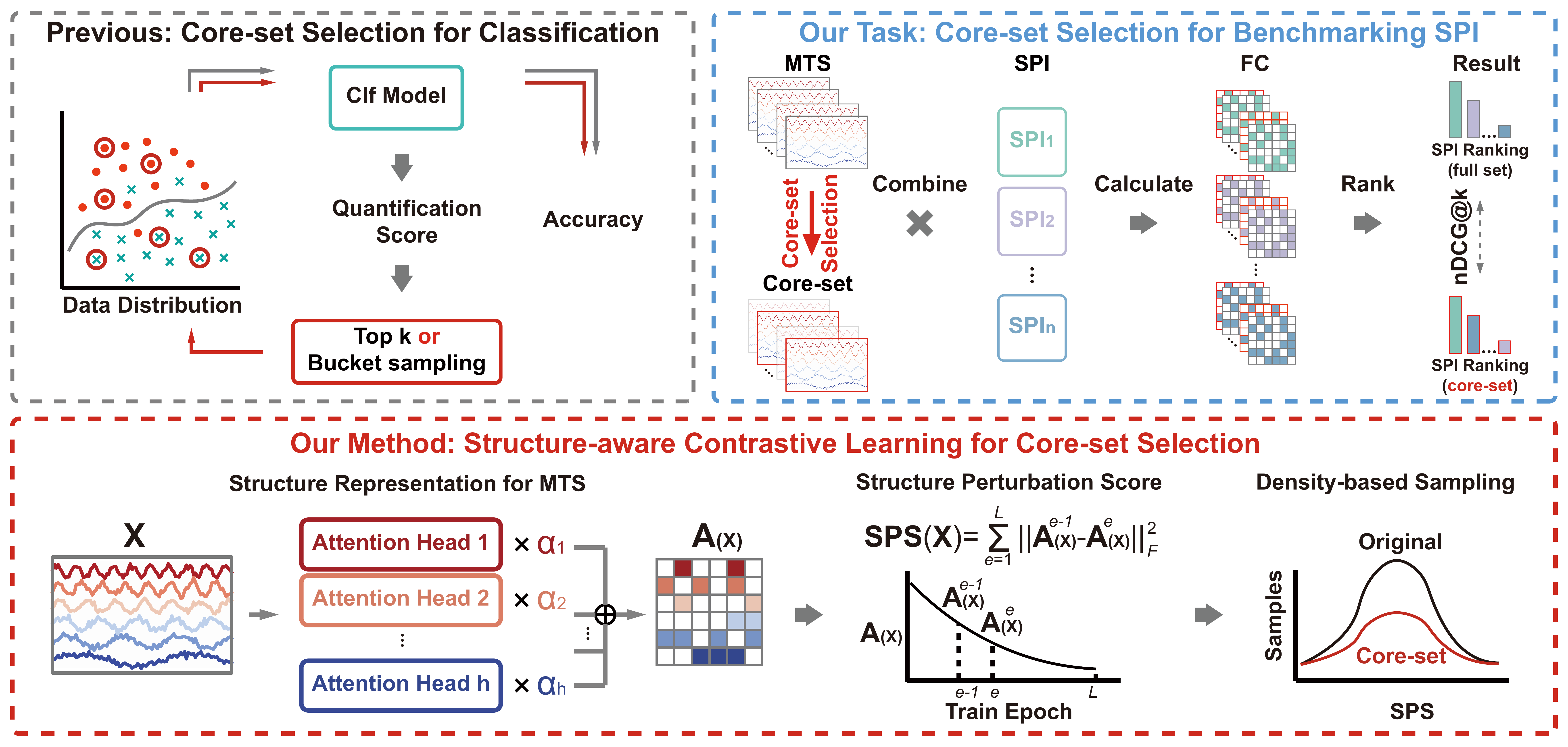}
  \caption{Overview of the \textbf{SCLCS} framework for ranking-preserving core-set selection.
Contrasting with selection for single-model classification (top left), our task is to preserve the performance ranking of SPIs (top right). Our method (bottom) achieves this using a Transformer to learn structures, our novel \textbf{SPS} metric to ensure stability, and a density-aware strategy to promote diversity.}
  \label{problem}
\end{figure}

Our theoretical analysis and empirical results on 130 candidate SPIs support our design choices and show consistent improvements over strong baselines. The main contributions are:
(1) We formulate core-set selection for efficient FC operator (SPI) benchmarking as a ranking-preservation problem.
(2) We propose \textbf{SCLCS}, a structure-aware framework for selecting stable and diverse samples for ranking-based benchmarking.
(3) We provide a universal approximation result for continuous SPI mappings (Theorem~\ref{thm:universal}) and introduce \textbf{SPS}, a new use of attention dynamics to quantify structural heterogeneity.
(4) We show that \textbf{SCLCS} enables reliable benchmarking at a fraction of the computational cost, making large-scale comparisons practical.

\section{Related Work}

\textbf{Benchmarking in Functional Connectivity.} Selecting an appropriate SPI (i.e., a network modeling method) is a central challenge in modern FC research. Prior studies show that different SPIs can yield divergent FC topologies and, consequently, different scientific conclusions~\citep{smith2011network,smith2013functional,bobadilla2020measures,mohanty2020rethinking,honari2021evaluating,luppi2024systematic}, contributing to long-standing concerns about reproducibility~\citep{open2015estimating,botvinik2020variability,marek2022reproducible}. Meanwhile, comprehensive libraries such as \texttt{pyspi}\citep{cliff2023unifying}, which include hundreds of SPIs, highlight the methodological richness of the field and magnify the scale of the selection problem\citep{liu2025benchmarking,roell2025measure}. These works motivate systematic benchmarking, but the computational cost of evaluating large SPI suites remains a key practical bottleneck. We address this bottleneck by introducing a core-set selection approach for efficient SPI benchmarking.

\textbf{Core-set Selection.} Core-set selection is a fundamental problem in machine learning~\citep{guo2022deepcore,feldman2020core,ros2019core}. Most score-based~\citep{Coleman2020Selection,feldman2020neural,paul2021deep,toneva2018an} and diversity-based~\citep{sener2018active,xia2022moderate} methods construct proxy datasets for training a single predictive model, which mismatches our ranking-preservation objective over many SPIs. Their criteria are often model-dependent (e.g., EVA~\citep{hong2024evolution}) and typically assume static i.i.d.\ inputs, overlooking the temporal dependencies in fMRI time series from which FC structure is derived. Consequently, selected core-sets may not transfer to ranking-based evaluation over large SPI suites~\citep{liu2024density,lee2024coreset}. Training-acceleration methods~\citep{hong2024diversified,killamsetty2021grad,mirzasoleiman2020coresets,wei2015submodularity} share the same single-model focus and thus do not directly address our setting.

To our knowledge, \textbf{SCLCS} is among the first methods tailored for accelerating benchmarking of FC operators (SPIs) via core-set selection.
It uses the stability of learned synchronization structures during training as a selection criterion, which is particularly natural for neuroimaging.
While related to graph structure learning~\citep{li2023gslb,zhou2023opengsl,zong2024new}, \textbf{SCLCS} treats learned structure as a diagnostic probe rather than an inference output.
We therefore do not review general-purpose structure learning in depth.

\section{Preliminaries}
\label{sec:preliminaries}
\paragraph{Benchmarking FC modeling.} The goal of FC benchmarking is to produce a principled ranking of statistical pairwise interaction (SPI) operators from a set $\mathcal{S}$. For a given fMRI dataset $\mathcal{X}$, where each sample $\mathbf{X} \in \mathcal{X}$ is a matrix in $\mathbb{R}^{N \times T}$, each operator $S \in \mathcal{S}$ maps $\mathbf{X}$ to an FC matrix $S(\mathbf{X}) \in \mathbb{R}^{N \times N}$. An evaluation index, $\mathcal{I}: \mathcal{S} \rightarrow \mathbb{R}$, assigns a score to each SPI under a chosen evaluation protocol. This process induces a ranking over all SPIs in $\mathcal{S}$, which we denote as $\text{Rank}(\mathcal{S}, \mathcal{X})$, to guide the selection of a suitable SPI for subsequent analysis.

\paragraph{Core-set Selection for Benchmarking FC Modeling.} Computing the full-dataset ranking $\text{Rank}(\mathcal{S}, \mathcal{X})$ is often computationally prohibitive. We therefore introduce the task of core-set selection for benchmarking, which seeks to identify a small subset of samples $\mathcal{X}_c\subset \mathcal{X}$ where $|\mathcal{X}_c|\ll |\mathcal{X}|$ that acts as an efficient proxy. Formally, a high-quality core-set is the solution to the following optimization problem:
\begin{equation}
\mathcal{X}_c^* = \operatorname*{argmin}_{\mathcal{X}' \subset \mathcal{X},\, |\mathcal{X}'|=c}
\mathcal{D}\!\bigl(\text{Rank}(\mathcal{S}, \mathcal{X}),\, \text{Rank}(\mathcal{S}, \mathcal{X}')\bigr).
\end{equation}
where $\mathcal{D}(\cdot, \cdot)$ is a ranking discrepancy metric (e.g., based on nDCG@k) and $c$ is the core-set budget. The goal is to preserve the full-dataset ranking while using only $c$ samples.

Directly optimizing this objective is intractable, as it requires exhaustively evaluating an exponential number of subsets, with each evaluation incurring the very computational cost we aim to avoid. We therefore propose a practical proxy: selecting a structurally representative subset. The core hypothesis is that preserving the distribution of functional connectivity structures also preserves the SPI ranking. Our \textbf{SCLCS} framework, detailed next, is designed to find such a subset.

\section{Method}

In this section, we introduce the detailed formulation of the proposed \textbf{SCLCS}. \textbf{SCLCS} consists of four modules: (1) attention-based FC learning, (2) structural perturbation score calculation, (3) structure-aware density-balanced sampling, and (4) contrastive learning.

\subsection{Attention-based FC Learning}
\label{sec:transformer}
To select a rank-preserving core-set, our framework first requires an encoder that can learn a general and expressive representation of each sample's FC structure. The self-attention mechanism within Transformers is a natural candidate for this task, as it can model complex inter-regional relationships~\citep{vaswani2017attention}. However, na\"{\i}ve fusion of multiple attention heads via uniform averaging is insufficient, as it can obscure distinct structural patterns learned by individual heads \textbf{Theorem~\ref{thm:interference}} (proved in \textbf{Appendix~\ref{app:interference}}).

\begin{theorem}[Interference of Averaged Attention]
\label{thm:interference}
Let $\{\mathbf{A}_h\}_{h=1}^H$ be row-stochastic attention matrices.
Assume disjoint structural masks: for each row $i$ there exist pairwise-disjoint sets
$\{S_h^{(i)}\}_{h=1}^H$ such that $\mathbf{A}_h(i,j)=0$ for all $j\notin S_h^{(i)}$.
Let $\bar{\mathbf{A}}:=\tfrac1H\sum_{h=1}^H \mathbf{A}_h$.
Then for every row $i$:
\[
\operatorname{supp}\!\bigl(\bar{\mathbf a}^{(i)}\bigr)=\bigcup_{h=1}^H S_h^{(i)}
\quad\text{and}\quad
\mathcal H\!\bigl(\bar{\mathbf a}^{(i)}\bigr)>\min_{1\le h\le H}\mathcal H\!\bigl(\mathbf a_h^{(i)}\bigr)
\ \text{if } \{\mathbf a_h^{(i)}\}_{h=1}^H \text{ are not all identical.}
\]
In particular, if $H\ge2$, naive averaging expands support beyond any single head's mask and inflates entropy,
blurring head-specific structure.
\end{theorem}

\textbf{Theorem~\ref{thm:interference}} pertains strictly to the internal attention maps and explains the empirical failure of directly applying the traditional Transformer on core-set selection for benchmarking FC modeling. It motivates us to propose an adaptive fusion mechanism that aggregates head-specific attention matrices via learnable weights. This modification is not merely an engineering choice: we prove it endows the architecture with the power of a universal approximator for the class of continuous FC operators as formalized in \textbf{Theorem~\ref{thm:universal}} (proved in \textbf{Appendix~\ref{app:universal}}). This provides a theoretical foundation for its ability to capture the diverse synchronization patterns required for our benchmarking task.

\begin{theorem}[Universal Approximation of Continuous Stochastic SPIs%
\footnote{This statement is for continuous targets on compact domains. If an SPI uses discrete thresholds (hard masks), the guarantee applies to any continuous relaxation (e.g., finite-temperature softmax / sigmoid gates) and then a limiting argument is required to justify the hard-threshold limit.}]
\label{thm:universal}
Let $\mathcal{X}\subset\mathbb{R}^{N\times T}$ be compact.
Let $S:\mathcal{X}\to\Delta^{N-1\times N}$ be continuous, where
$\Delta^{N-1\times N}:=\{\mathbf{P}\in\mathbb{R}^{N\times N}:\mathbf{P}\mathbf{1}=\mathbf{1},\ \mathbf{P}\ge 0\}$ denotes the set of row-stochastic matrices.
Consider the adaptive multi-head attention family
\begin{equation}
\label{eq:universal_family}
\mathbf{A}_\theta(\mathbf{X})
\;:=\;
\sum_{h=1}^{H}\alpha_h\,
\operatorname{softmax}\!\Bigl(
\tfrac{\mathbf{X}\mathbf{W}_h^{Q}(\mathbf{W}_h^{K})^{\!\top}\mathbf{X}^{\!\top}}{\tau}
\Bigr),
\qquad
\boldsymbol{\alpha}\in\Delta^{H-1},\ \tau>0,
\end{equation}
where $\operatorname{softmax}$ is applied row-wise.
Then for every $\varepsilon>0$ there exist $H$, $\tau$, and parameters
$\{\mathbf{W}_h^{Q},\mathbf{W}_h^{K}\}_{h=1}^{H}$ and $\boldsymbol{\alpha}$
such that
\begin{equation}
    \sup_{\mathbf{X}\in\mathcal{X}}
    \bigl\|\mathbf{A}_\theta(\mathbf{X})-S(\mathbf{X})\bigr\|_{F}
    <\varepsilon.
\end{equation}

\end{theorem}

Our implementation is as follows: for each fMRI sample \( \mathbf{X} \in \mathbb{R}^{N \times T} \), we treat the $N$ ROIs as input tokens, where each token has a feature dimension of $T$. Each attention head independently projects queries and keys using learnable linear maps, parameterized by matrices \(\mathbf{W}_h^Q\) and \(\mathbf{W}_h^K\):
\begin{equation}
    \mathbf{Q}_h = \mathbf{X} \mathbf{W}_h^Q, \quad
\mathbf{K}_h = \mathbf{X} \mathbf{W}_h^K, \quad
\mathbf{W}_h^Q, \mathbf{W}_h^K \in \mathbb{R}^{D \times d}, \quad h=1,\dots,H.
\end{equation}
The attention matrix from head \( h \) is computed as:
\begin{equation}
    \mathbf{A}_h = \text{softmax} \left( \frac{\mathbf{Q}_h \mathbf{K}_h^\top}{\sqrt{d}} \right), \quad \mathbf{A}_h \in \mathbb{R}^{N \times N}.
\end{equation}

Motivated by \textbf{Theorem~\ref{thm:interference}}, uniform head averaging is a
structural constraint: $\alpha_h\equiv 1/H$ forces every sample to use the same
(high-entropy) centroid of head-wise attention patterns.
By strict concavity of Shannon entropy, this averaging inflates uncertainty and
smears head-specific structure.
Learnable fusion weights $\boldsymbol\alpha$ relax the constraint, enabling
sparse/peaked mixtures (up to single-head selection) to reduce interference
while still combining complementary patterns.
Thus the operator class strictly expands: by \textbf{Theorem~\ref{thm:universal}},
an adaptive fusion module can approximate continuous FC operator on compact
domains.

Thus, we propose a learnable fusion mechanism that aggregates head-specific attention matrices via adaptive weights, formulated as:
\begin{equation}
\label{eq:adaptive_attn}
\mathbf{A} = \sum_{h=1}^H \boldsymbol{\alpha}_h \mathbf{A}_h, \quad \text{with} \quad \sum_{h=1}^H \boldsymbol{\alpha}_h = 1, \quad \boldsymbol{\alpha}_h \ge 0,
\end{equation}
where the weight \( \boldsymbol{\alpha} \) is normalized via \texttt{softmax}. The resulting matrix \( \mathbf{A} \in \mathbb{R}^{N \times N} \) serves as the operational definition of a sample's FC structure, forming the basis for our selection criteria. Importantly, we treat these attention maps as a normalized structural probe---a sample-specific proxy for synchronization structure---not as an attempt to replicate the raw outputs of any particular SPI.

\subsection{Structural Perturbation Score (SPS)}
\label{sec:sps}

Having established an encoder that is expressive enough to capture diverse FC structures, the next challenge is to define a criterion for identifying the most fundamental samples for a robust benchmark. Our central hypothesis is that samples representing common, foundational connectivity patterns will induce stable structural representations during training, while noisy or atypical samples will cause greater fluctuations. To quantify this phenomenon, we propose the \textbf{S}tructural \textbf{P}erturbation \textbf{S}core (\textbf{SPS}), a metric grounded in the principle that perturbation magnitude reflects structural heterogeneity (\textbf{Proposition~\ref{prop:mixture_sps}}). 

\begin{proposition}[Mixture-driven perturbation magnitude]
\label{prop:mixture_sps}
Let $S^{(1)},\dots,S^{(K)}\in\mathbb R^{N\times N}$ be distinct prototypes and
$D_{kl}:=\|S^{(k)}-S^{(l)}\|_F^2$ ($D_{kl}>0$ for $k\neq l$).
Let $(Z_e)_{e\ge 1}$ be i.i.d.\ with $\Pr[Z_e=S^{(k)}]=\lambda_k$,
$\sum_k\lambda_k=1$, and define $\Delta_e:=\|Z_e-Z_{e-1}\|_F^2$.

Then
\begin{equation}
\label{eq:mixture_exact}
\mathbb E[\Delta_e]
=\sum_{k,l}\lambda_k\lambda_l D_{kl}
=2\sum_{k<l}\lambda_k\lambda_l D_{kl}.
\end{equation}
With $D_{\min}:=\min_{k<l}D_{kl}$ and $D_{\max}:=\max_{k<l}D_{kl}$,
\begin{equation}
\label{eq:mixture_bounds}
D_{\min}\Bigl(1-\sum_k\lambda_k^2\Bigr)
\le \mathbb E[\Delta_e]
\le
D_{\max}\Bigl(1-\sum_k\lambda_k^2\Bigr).
\end{equation}
In particular, $\mathbb E[\Delta_e]$ scales with the Gini impurity
$1-\sum_k\lambda_k^2$ up to constants set by prototype separation.
If $D_{kl}\equiv D$ for all $k\neq l$, then
\begin{equation}
\label{eq:mixture_isotropic}
\mathbb E[\Delta_e]=D\Bigl(1-\sum_k\lambda_k^2\Bigr).
\end{equation}
\textup{(Proof in Appendix~\ref{app:mixture_sps}.)}
\end{proposition}

\textbf{Proposition~\ref{prop:mixture_sps}} indicates samples that are a purer representation of a single archetype will be more stable (low \textbf{SPS}). The \textbf{SPS} for a sample $\mathbf{X} \in \mathcal{X}$ is thus defined as the cumulative structural instability across $L$ training epochs:

\begin{equation}
\label{eq:sps_epoch}
\mathrm{SPS}(\mathbf{X}) = \frac{1}{L} \sum_{e=1}^{L} \left\| \mathbf{A}^{(e)}_{(\mathbf{X})} - \mathbf{A}^{(e-1)}_{(\mathbf{X})} \right\|_F^2,
\end{equation}

where \( \mathbf{A}^{(e)}_{(\mathbf{X})} \) denotes the attention-based structure matrix of sample \( \mathbf{X} \in \mathcal{X} \) at training epoch \( e \in \{1, \ldots, L\} \) and $||\cdot||_{F}$ denotes the Frobenius norm. Specifically, $\mathbf{A}^{(e)}_{(\mathbf{X})}$ is a proxy for structural representation, not a proposed network model. In this way, \textbf{SPS} captures the structural volatility of the sample-specific synchronization graph during training, not fidelity to any specific SPI. 
Our rationale is that a robust and reliable benchmark is built upon foundational, structurally stable samples. As supported by \textbf{Proposition~\ref{prop:mixture_sps}}, low-\textbf{SPS} samples exhibit less internal structural conflict and thus represent stable archetypes of functional connectivity. Therefore, our primary selection strategy is to rank samples by their \textbf{SPS} and select those with the lowest \textbf{SPS}.

For \textbf{SPS} to be a reliable metric, however, we must ensure that it is a consistent estimator that does not depend on the arbitrary length of the training process. \textbf{Lemma~\ref{lem:sps_consistency}} (in \textbf{Appendix~\ref{app:consist}}) provides this theoretical guarantee: Under stable perturbation dynamics, the assumptions of \textbf{Lemma~\ref{lem:sps_consistency}} are satisfied. 
We used extensive grid search to find a configuration that achieves optimal performance on the downstream ranking preservation task as detailed in \textbf{Appendix~\ref{app:repro}}.
Notably, the assumptions of that configuration's stationarity and ergodicity in \textbf{Lemma~\ref{lem:sps_consistency}} are empirically supported by our convergence analysis in \textbf{Appendix~\ref{app:staconv}}, where we show that the perturbation dynamics stabilize as the model converges. This standard procedure sufficiently validates the feasibility of \textbf{SPS}.

\subsection{Structure-aware Density-Balanced Sampling}
\label{sec:structure_aware_sampling}

While selecting for structurally stable (low-\textbf{SPS}) samples provides a robust foundation, a na\"{\i}ve top-$k$ selection risks creating a core-set with low diversity by over-selecting from dense clusters of typical patterns. This lack of diversity can cause the core-set benchmark to diverge from the full-dataset ranking (\textbf{Theorem~\ref{thm:unbalanced_topk}}).

\begin{theorem}[Persistent bias of top-$k$ selection]
\label{thm:unbalanced_topk}
Let $\mathcal X$ contain two clusters $C_p,C_q$ with proportions $\pi_p,\pi_q$.
Given $\mathbf x\in C_r$ ($r\in\{p,q\}$), let the score $s(\mathbf x)$ have
continuous CDF $F_r$, and assume scores are independent across samples.
Select the $k=\lfloor \rho N\rfloor$ samples with the smallest scores ($\rho\in(0,1)$),
and write $\widehat\pi_r:=|S_k\cap C_r|/k$.

Let $\tau$ satisfy the mixture-quantile equation
\begin{equation}
\label{eq:tau_def}
\pi_p F_p(\tau)+\pi_q F_q(\tau)=\rho,
\end{equation}
and assume strict separation at $\tau$:
\begin{equation}
\label{eq:cdf_gap}
\gamma:=F_p(\tau)-F_q(\tau)>0.
\end{equation}
Then
\begin{equation}
\widehat\pi_p \xrightarrow{\Pr} \frac{\pi_p F_p(\tau)}{\rho}
      = \pi_p + \delta,
\qquad
\widehat\pi_q \xrightarrow{\Pr} \frac{\pi_q F_q(\tau)}{\rho}
      = \pi_q - \delta,
\qquad
\delta:=\frac{\pi_p\pi_q}{\rho}\gamma>0.
\label{eq:limit_bias}
\end{equation}
Consequently, the representation error
$\Delta_k:=|\widehat\pi_p-\pi_p|+|\widehat\pi_q-\pi_q|$
satisfies $\Delta_k\xrightarrow{\Pr}2\delta>0$.
\textup{(See \textbf{Appendix~\ref{app:proof_unbalanced_topk}}.)}
\end{theorem}

To explicitly balance stability with diversity, we introduce a density-aware sampling scheme, yielding the \textbf{SCLCS}\textsubscript{Dense} variant. This scheme first ensures robustness by retaining a pool of the most stable samples (the bottom $1-\beta$ quantile of \textbf{SPS} scores), then promotes diversity by applying Kernel Density Estimation (KDE)~\citep{wkeglarczyk2018kernel} to up-weight samples from sparser regions within that stable pool. This mitigates redundancy and ensures the core-set captures a broader range of structurally distinct subtypes, which is crucial for including less common but potentially critical neural patterns often associated with clinical biomarkers.

Specifically, given the set of \textbf{SPS} for all samples in $\mathcal{X}$, we first discard the top $\beta$ quantile of the most unstable samples to form a stable candidate pool $\tilde{\mathcal{X}}$. This is formally defined as:
\begin{equation}
\tilde{\mathcal{X}} = \{ \mathbf{X} \in \mathcal{X} \mid \textbf{SPS}(\mathbf{X}) \leq Q_{1-\beta} \},
\end{equation}
where \( Q_{1-\beta} \) is the empirical quantile. On $\tilde{\mathcal{X}}$, we fit a Gaussian KDE to the empirical distribution of
$\{\mathrm{SPS}(\mathbf{X}) : \mathbf{X}\in\tilde{\mathcal{X}}\}$, and define the
local density of a sample as the KDE evaluated at its \textbf{SPS}:
\begin{equation}
\hat{p}_{\mathrm{SPS}}(s) = \texttt{KDE}\bigl(\{\mathrm{SPS}(\mathbf{X})\}_{\mathbf{X}\in\tilde{\mathcal{X}}}\bigr), 
\qquad
\rho(\mathbf{X}) = \hat{p}_{\mathrm{SPS}}\!\bigl(\mathrm{SPS}(\mathbf{X})\bigr).
\end{equation}
To promote diversity, weights are set inversely proportional to \( \rho(\mathbf{X}) \) and normalized over $\tilde{\mathcal{X}}$:
\begin{equation}
w(\mathbf{X}) = \frac{1}{\rho(\mathbf{X}) + \epsilon}, \qquad
w(\mathbf{X}) \leftarrow \frac{w(\mathbf{X})}{\sum_{\mathbf{X}'\in\tilde{\mathcal{X}}} w(\mathbf{X}')}.
\end{equation}

We then select \( m \) samples without replacement from \( \tilde{\mathcal{X}} \) using weights \( \{w(\mathbf{X})\} \). This yields a structurally diverse subset that up-weights samples in low-density regions of the \textbf{SPS} distribution within the stable pool. Theoretical guarantees on coverage and benchmarking consistency are demonstrated in \textbf{Theorem~\ref{thm:eps_cover}} and \textbf{Theorem~\ref{thm:discrep}}, respectively (provided and proved in \textbf{Appendix~\ref{app:cover}} and \textbf{Appendix~\ref{app:discrep}}, respectively). Notably, this scheme is applicable to other score-based methods by replacing \textbf{SPS} to different metrics as shown in \textbf{Appendix~\ref{app:density-bal}}.

\subsection{Structure-aware Contrastive Learning}
\label{sec:contrastive}
To learn structural representations, we train the encoder with a structure-aware contrastive objective.
Motivated by FC evaluation practices that exploit inter-subject differences~\citep{liu2025benchmarking,luppi2024systematic}, we enforce consistency among samples drawn from the same subject within a scan session.
This identity-supervised setup encourages the model to capture stable, person-specific traits (“brain fingerprints”~\citep{lu2024brain}), providing a task-agnostic signal for structure-based selection.

First, to obtain a graph-level embedding for each sample, we compute node-level embeddings $\mathbf{Z} \in \mathbb{R}^{N \times d}$ by applying the learned attention matrix to the value embeddings and projecting the result through a final linear layer. A global mean pooling is then applied to obtain the graph-level embedding $\mathbf{z} \in \mathbb{R}^{d}$, which captures the sample's global topological semantics and serves as input to our contrastive loss:
\begin{equation}
\label{eq:contrastive}
\mathcal{L}_{\text{contrast}} = \frac{1}{|\mathcal{P}|} \sum_{(i,j) \in \mathcal{P}}
- \log \frac{
\exp\left( \text{sim}(\mathbf{z}_i, \mathbf{z}_j) / \tau \right)
}{
\sum_{k \in \mathcal{N}(i)} \exp\left( \text{sim}(\mathbf{z}_i, \mathbf{z}_k) / \tau \right)
},
\end{equation}
where $\text{sim}(\cdot,\cdot)$ is cosine similarity, $\tau$ is a temperature parameter, and the pairs are defined across a batch. Positive pairs $(i,j) \in \mathcal{P}$ consist of two different temporal segments from the same subject, while for a given anchor sample $i$, the negative sample $k \in \mathcal{N}(i)$ are all other samples in the batch from different subjects. All the trainable parameters are optimized using Adam~\citep{kingma2014adam}.

\section{Experiment}\label{sec:exp}
In this section, we present empirical results to validate our proposed framework. We first detail the experimental settings and then present the main quantitative and qualitative results comparing \textbf{SCLCS} to SOTA baselines. Due to space constraints, several important settings and supplementary analyses, including detailed experimental settings for reproduction (\textbf{Appendix~\ref{app:repro}}), a detailed computational cost breakdown (\textbf{Appendix~\ref{app:comps}}), the effect of supervised information on baselines (\textbf{Appendix~\ref{app:labelef}}), the application of our density-aware sampling to other baselines (\textbf{Appendix~\ref{app:density-bal}}), empirical results (\textbf{Appendix~\ref{app:theorem_exp}}) for supporting \textbf{Theorem~\ref{thm:universal}} and the assumptions in \textbf{Lemma~\ref{lem:sps_consistency}}, and additional generalization and robustness analysis (\textbf{Appendix~\ref{app:GandR}}).

\subsection{Experimental Settings}\label{settings}
\paragraph{Data}
We validate our framework on the REST-meta-MDD dataset~\citep{yan2019reduced}, a large-scale, multi-site resting-state fMRI collection comprising 1,642 subjects from 17 sites. This highly heterogeneous collection was released with a standardized preprocessing pipeline, providing a rigorous testbed for core-set selection. For efficiency and consistency with prior work, we focus on a subset of 904 subjects~\citep{long2020altered}. To capture dynamic patterns, each subject's fMRI record is segmented into overlapping temporal samples via a sliding window, yielding 4,520 samples. Demographic statistics are in Table~\ref{tab:site-summary}. Complete data and preprocessing details are in \textbf{Appendix~\ref{app:dp}}.

\begin{table}[!h]
\centering
\caption{Summary of the used subset of REST-meta-MDD.}
\label{tab:site-summary}
\resizebox{0.9\textwidth}{!}{
\begin{tabular}{c|c|c|c|c|c|c|c}
\toprule
Site & \#Samples &  \#HC &  \#MDD &  \#Male &  \#Female & Age Range & Education Range (Years) \\
\midrule
15   & 335 &   37 &    30 &     26 &       41 &    19--65 &                   5--21 \\
17   & 410 &  41 &    41 &     27 &       55 &    18--30 &                   9--17 \\
19   & 245 &  31 &    18 &     19 &       30 &    18--51 &                   5--15 \\
20   & 2395 & 229 &   250 &    157 &      322 &    18--65 &                   3--20 \\
21   & 720 &  65 &    79 &     62 &       82 &    18--65 &                   5--15 \\
22   & 190 &  20 &    18 &     21 &       17 &    19--47 &                   8--17 \\
23   & 225 &  23 &    22 &     18 &       27 &    19--54 &                   6--20 \\
Overall & 4520 & 458 & 446 & 330 & 574 & 18--65 & 3--21\\
\bottomrule
\end{tabular}
}
\end{table}

\paragraph{Baselines}
Following the experimental setup in~\citet{hong2024evolution}, we extend their comparison with two additions (k-Means and BOSS), evaluating against 9 baselines in total: (1) Random; (2) k-Means~\citep{hartigan1979algorithm}; (3) Forgetting score~\citep{toneva2018an}; (4) Entropy~\citep{Coleman2020Selection}; (5) EL2N~\citep{paul2021deep}; (6) AUM~\citep{pleiss2020identifying}; (7) CCS~\citep{zheng2023coveragecentric}; (8)EVA~\citep{hong2024evolution};(9) BOSS~\citep{acharya2024balancing}. The latter $7$ of them are SOTA methods designed for core-set selection. Detailed introduction is summarized in \textbf{Appendix~\ref{app:bias}}.

\paragraph{Environment}
Experiments are performed on an 8-GPU (H20) high-performance computing cluster provided by the Large-scale Instrument Sharing Platform of Southwest University.

\paragraph{Evaluation Protocol}
Our goal is to select a subset that preserves model ranking, not to train a single predictive model. Therefore, our evaluation deviates from the standard train/test split. The protocol is as follows: (1) Each method selects a core-set of a given size from the entire dataset. (2) We then compute the SPI performance ranking on both the full dataset (ground truth) and the selected core-set. (3) The quality of the core-set is measured by the consistency between these two rankings.

\paragraph{Task}
We use two distinct downstream tasks to evaluate SPI discriminability: brain fingerprinting (distinguishing individuals based on subject ID), which probes for fine-grained, subject-specific structures, and MDD diagnosis, which relies on cohort-level patterns.

\paragraph{Metrics}
We use two primary metrics. (1) \textbf{Discriminability Score}, a metric based on Spearman's rank correlation~\citep{sedgwick2014spearman} that quantifies the class separability (within- vs. between-class) of the resulting FC matrices. (2) \textbf{Ranking Consistency}, the concordance between the full-dataset and core-set SPI rankings using nDCG@\(5/10/20\)~\citep{wang2013theoretical}.

\subsection{Quantitative Results}\label{quanti}
We present the primary quantitative results in Table~\ref{tab:bfrt} and Table~\ref{tab:MRT}. The evaluation includes \textbf{SCLCS} (using low-\textbf{SPS} top-$k$ selection), \textbf{SCLCS}\textsubscript{Dense} (density-aware selection), and \textbf{SPS}\textsubscript{MHA} (a variant using na\"{\i}ve attention averaging to empirically test \textbf{Theorem~\ref{thm:interference}}).

\paragraph{Brain Fingerprinting}

\begin{table}[!h]
\centering
\caption{Performance of different methods on brain fingerprinting ranking task (mean $\pm$ std) and nDCG@k is reported as percentage ($\times$100).}\label{tab:bfrt}
\resizebox{\textwidth}{!}{
\begin{tabular}{l|ccc|ccc|ccc}
\toprule
\multirow{2}{*}{Method} & \multicolumn{3}{c|}{\textbf{nDCG@5}} & \multicolumn{3}{c|}{\textbf{nDCG@10}} & \multicolumn{3}{c}{\textbf{nDCG@20}} \\ \cmidrule(lr){2-10}
& 0.1 & 0.3 & 0.5 & 0.1 & 0.3 & 0.5 & 0.1 & 0.3 & 0.5 \\
\midrule
Random & 15.17$_{\pm\text{13.97}}$ & \underline{69.57}$_{\pm\text{24.69}}$ & \underline{66.60}$_{\pm\text{20.21}}$ & 17.97$_{\pm\text{17.50}}$ & \underline{71.19}$_{\pm\text{20.93}}$ & \underline{66.98}$_{\pm\text{19.47}}$ & 21.97$_{\pm\text{21.91}}$ & \underline{71.63}$_{\pm\text{21.27}}$ & \underline{68.13}$_{\pm\text{17.80}}$ \\ 
k-Means & 17.32$_{\pm\text{11.53}}$ & 65.72$_{\pm\text{12.25}}$ & 67.29$_{\pm\text{11.41}}$ & 21.23$_{\pm\text{15.32}}$ & 64.35$_{\pm\text{14.83}}$ & 63.37$_{\pm\text{9.92}}$ & 22.15$_{\pm\text{15.19}}$ & 62.58$_{\pm\text{19.22}}$ & 66.78$_{\pm\text{15.28}}$ \\\midrule
Forgetting & 14.41$_{\pm\text{7.84}}$ & 54.43$_{\pm\text{13.26}}$ & 43.36$_{\pm\text{3.35}}$ & 15.49$_{\pm\text{7.83}}$ & 48.57$_{\pm\text{7.85}}$ & 49.87$_{\pm\text{4.66}}$ & 20.23$_{\pm\text{7.21}}$ & 55.80$_{\pm\text{6.72}}$ & 49.70$_{\pm\text{4.60}}$ \\
Entropy & 47.40$_{\pm\text{40.26}}$ & 22.84$_{\pm\text{13.20}}$ & 59.95$_{\pm\text{26.99}}$ & 37.73$_{\pm\text{25.94}}$ & 32.05$_{\pm\text{15.22}}$ & 58.68$_{\pm\text{23.71}}$ & 36.05$_{\pm\text{18.11}}$ & 35.90$_{\pm\text{16.19}}$ & 57.72$_{\pm\text{21.84}}$ \\
El2N & 35.56$_{\pm\text{36.16}}$ & 20.30$_{\pm\text{5.77}}$ & 31.03$_{\pm\text{34.57}}$ & 36.51$_{\pm\text{25.38}}$ & 33.18$_{\pm\text{14.68}}$ & 32.70$_{\pm\text{30.84}}$ & 33.30$_{\pm\text{21.55}}$ & 35.82$_{\pm\text{12.10}}$ & 40.06$_{\pm\text{32.08}}$ \\
AUM & \underline{65.92}$_{\pm\text{33.80}}$ & 56.68$_{\pm\text{11.11}}$ & 38.17$_{\pm\text{13.94}}$ & \underline{60.95}$_{\pm\text{30.91}}$ & 62.05$_{\pm\text{4.91}}$ & 36.83$_{\pm\text{8.78}}$ & \underline{51.75}$_{\pm\text{22.42}}$ & 59.09$_{\pm\text{5.72}}$ & 38.34$_{\pm\text{6.88}}$ \\
CCS & 1.90$_{\pm\text{2.07}}$ & 30.53$_{\pm\text{14.15}}$ & 46.65$_{\pm\text{22.32}}$ & 2.92$_{\pm\text{3.24}}$ & 29.13$_{\pm\text{8.00}}$ & 51.78$_{\pm\text{24.12}}$ & 16.24$_{\pm\text{13.34}}$ & 32.56$_{\pm\text{14.15}}$ & 52.18$_{\pm\text{20.20}}$ \\
EVA & 38.40$_{\pm\text{40.57}}$ & 62.03$_{\pm\text{26.64}}$ & 37.80$_{\pm\text{42.45}}$ & 43.37$_{\pm\text{19.92}}$ & 55.01$_{\pm\text{20.05}}$ & 49.56$_{\pm\text{33.28}}$ & 43.22$_{\pm\text{15.28}}$ & 53.51$_{\pm\text{14.70}}$ & 65.49$_{\pm\text{21.99}}$ \\
BOSS & 15.98$_{\pm\text{25.58}}$ & 42.11$_{\pm\text{24.33}}$ & 35.36$_{\pm\text{9.21}}$ & 29.44$_{\pm\text{11.05}}$ & 40.57$_{\pm\text{23.37}}$ & 39.15$_{\pm\text{6.71}}$ & 31.45$_{\pm\text{9.45}}$ & 38.24$_{\pm\text{19.65}}$ & 38.92$_{\pm\text{5.97}}$ \\ \midrule
\textbf{SCLCS} & \textbf{81.21}$_{\pm\text{2.86}}$ & 50.24$_{\pm\text{13.16}}$ & \textbf{72.68}$_{\pm\text{20.83}}$ & \textbf{66.54}$_{\pm\text{1.10}}$ & 49.45$_{\pm\text{14.18}}$ & \textbf{71.86}$_{\pm\text{4.35}}$ & \textbf{57.46}$_{\pm\text{0.52}}$ & 53.40$_{\pm\text{16.27}}$ & \textbf{70.13}$_{\pm\text{1.57}}$ \\
\textbf{SCLCS}$_\text{Dense}$ & 35.73$_{\pm\text{31.18}}$ & \textbf{79.18}$_{\pm\text{6.29}}$ & 51.54$_{\pm\text{13.83}}$ & 35.84$_{\pm\text{28.43}}$ & \textbf{73.45}$_{\pm\text{1.42}}$ & 50.54$_{\pm\text{15.16}}$ & 41.03$_{\pm\text{35.89}}$ & \textbf{72.96}$_{\pm\text{3.35}}$ & 55.43$_{\pm\text{13.64}}$ \\
\textbf{SPS}$_\text{MHA}$ & 1.32$_{\pm\text{2.07}}$ & 12.23$_{\pm\text{5.11}}$ & 15.62$_{\pm\text{6.33}}$ & 2.92$_{\pm\text{1.13}}$ & 13.12$_{\pm\text{11.33}}$ & 11.28$_{\pm\text{4.32}}$ & 1.21$_{\pm\text{1.04}}$ & 12.13$_{\pm\text{3.17}}$ & 12.18$_{\pm\text{7.23}}$ \\
\bottomrule
\end{tabular}
}
\end{table}

This task rewards subject-specific patterns, which aligns with our identity-supervised objective. As shown in Table~\ref{tab:bfrt}, \textbf{SCLCS} achieves stronger performance with lower variance at sampling ratios 0.1 and 0.5, suggesting that selecting structurally stable samples via low-\textbf{SPS} top-$k$ ranking is effective. In contrast, \textbf{SPS}\textsubscript{MHA} performs poorly, consistent with \textbf{Theorem~\ref{thm:interference}}. At the moderate ratio 0.3, \textbf{SCLCS} degrades, whereas \textbf{SCLCS}\textsubscript{Dense} performs best. This pattern supports \textbf{Theorem~\ref{thm:unbalanced_topk}}: when stability ranking alone is insufficient, explicitly promoting structural diversity provides a corrective signal.

\paragraph{MDD Diagnosis}
\begin{table}[ht]
\centering
\caption{Performance of different methods on MDD diagnosis ranking task (mean $\pm$ std) and nDCG@k is reported as percentage ($\times$100).}\label{tab:MRT}
\resizebox{\textwidth}{!}{
\begin{tabular}{l|ccc|ccc|ccc}
\toprule
\multirow{2}{*}{Method} & \multicolumn{3}{c|}{\textbf{nDCG@5}} & \multicolumn{3}{c|}{\textbf{nDCG@10}} & \multicolumn{3}{c}{\textbf{nDCG@20}} \\
\cmidrule(lr){2-10}
& 0.1 & 0.3 & 0.5 & 0.1 & 0.3 & 0.5 & 0.1 & 0.3 & 0.5 \\
\midrule
Random & 49.58$_{\pm\text{33.78}}$ & 23.77$_{\pm\text{35.85}}$ & 31.76$_{\pm\text{15.83}}$ & \underline{60.50}$_{\pm\text{33.36}}$ & 25.81$_{\pm\text{36.19}}$ & 77.95$_{\pm\text{11.77}}$ & 67.04$_{\pm\text{27.03}}$ & 30.30$_{\pm\text{39.67}}$ & 82.71$_{\pm\text{8.59}}$ \\
k-Means & 51.32$_{\pm\text{17.47}}$ & 32.45$_{\pm\text{33.46}}$ & 37.33$_{\pm\text{17.20}}$ & 30.18$_{\pm\text{12.65}}$ & 40.47$_{\pm\text{14.81}}$ & 43.75$_{\pm\text{22.84}}$ & 28.62$_{\pm\text{14.33}}$ & 37.87$_{\pm\text{14.77}}$ & 79.83$_{\pm\text{5.24}}$ \\\midrule
Forgetting & 28.70$_{\pm\text{28.06}}$ & 41.60$_{\pm\text{28.08}}$ & 61.42$_{\pm\text{4.31}}$ & 40.56$_{\pm\text{28.06}}$ & 50.27$_{\pm\text{32.37}}$ & 66.72$_{\pm\text{8.60}}$ & 44.37$_{\pm\text{28.07}}$ & 57.02$_{\pm\text{33.39}}$ & 75.70$_{\pm\text{7.58}}$ \\
Entropy & 29.00$_{\pm\text{45.10}}$ & 31.84$_{\pm\text{46.19}}$ & 57.48$_{\pm\text{29.81}}$ & 29.17$_{\pm\text{43.24}}$ & 40.75$_{\pm\text{44.05}}$ & 63.23$_{\pm\text{23.44}}$ & 30.30$_{\pm\text{43.52}}$ & 45.85$_{\pm\text{45.00}}$ & 70.79$_{\pm\text{18.00}}$ \\
El2N & 30.93$_{\pm\text{41.17}}$ & 51.70$_{\pm\text{28.14}}$ & 68.85$_{\pm\text{12.46}}$ & 36.05$_{\pm\text{24.41}}$ & 58.54$_{\pm\text{27.55}}$ & 72.96$_{\pm\text{15.77}}$ & 41.04$_{\pm\text{39.61}}$ & 64.68$_{\pm\text{26.25}}$ & 78.85$_{\pm\text{14.56}}$ \\
AUM & 36.07$_{\pm\text{12.59}}$ & 42.68$_{\pm\text{46.16}}$ & 35.49$_{\pm\text{14.37}}$ & 39.17$_{\pm\text{14.58}}$ & 44.41$_{\pm\text{44.05}}$ & 58.14$_{\pm\text{16.54}}$ & 44.94$_{\pm\text{18.48}}$ & 48.44$_{\pm\text{42.48}}$ & 64.36$_{\pm\text{14.40}}$ \\
CCS & \underline{55.95}$_{\pm\text{16.28}}$ & 59.92$_{\pm\text{30.08}}$ & 74.73$_{\pm\text{12.46}}$ & 59.25$_{\pm\text{16.26}}$ & 67.01$_{\pm\text{24.61}}$ & 75.41$_{\pm\text{18.08}}$ & \underline{68.98}$_{\pm\text{15.67}}$ & 71.62$_{\pm\text{21.15}}$ & 78.77$_{\pm\text{13.49}}$ \\
EVA & 31.80$_{\pm\text{17.70}}$ & 66.81$_{\pm\text{9.12}}$ & 70.07$_{\pm\text{6.05}}$ & 36.11$_{\pm\text{14.79}}$ & 72.61$_{\pm\text{8.73}}$ & 76.26$_{\pm\text{6.17}}$ & 48.39$_{\pm\text{19.48}}$ & 75.34$_{\pm\text{8.48}}$ & 81.73$_{\pm\text{5.61}}$ \\
BOSS & 42.44$_{\pm\text{24.37}}$ & 57.52$_{\pm\text{20.11}}$ & \underline{79.57}$_{\pm\text{16.62}}$ & 50.64$_{\pm\text{25.12}}$ & 64.86$_{\pm\text{21.86}}$ & \underline{84.70}$_{\pm\text{14.07}}$ & 58.11$_{\pm\text{23.22}}$ & 71.36$_{\pm\text{18.94}}$ & \underline{88.95}$_{\pm\text{9.49}}$ \\ \midrule
\textbf{SCLCS} & 48.38$_{\pm\text{23.00}}$ & \underline{70.27}$_{\pm\text{23.27}}$ & 64.18$_{\pm\text{25.98}}$ & 50.59$_{\pm\text{21.72}}$ & \underline{73.86}$_{\pm\text{15.67}}$ & 66.15$_{\pm\text{21.96}}$ & 61.12$_{\pm\text{16.75}}$ & \underline{76.89}$_{\pm\text{9.86}}$ & 68.43$_{\pm\text{17.52}}$ \\
\textbf{SCLCS}$_\text{Dense}$ & \textbf{57.29}$_{\pm\text{24.07}}$ & \textbf{74.62}$_{\pm\text{18.02}}$ & \textbf{81.87}$_{\pm\text{9.68}}$ & \textbf{64.34}$_{\pm\text{16.97}}$ & \textbf{77.52}$_{\pm\text{17.18}}$ & \textbf{86.25}$_{\pm\text{7.22}}$ & \textbf{69.84}$_{\pm\text{14.71}}$ & \textbf{82.70}$_{\pm\text{11.04}}$ & \textbf{89.45}$_{\pm\text{6.81}}$ \\
\textbf{SPS}$_\text{MHA}$ &19.13$_{\pm\text{15.36}}$ & 26.82$_{\pm\text{16.34}}$ & 27.45$_{\pm\text{19.31}}$ & 20.13$_{\pm\text{12.27}}$ & 22.75$_{\pm\text{14.33}}$ & 23.19$_{\pm\text{13.22}}$ & 25.73$_{\pm\text{11.92}}$ & 17.85$_{\pm\text{12.03}}$ & 20.73$_{\pm\text{14.90}}$ \\
\bottomrule
\end{tabular}
}
\end{table}

This cohort-level task requires broader structural coverage than fingerprinting. As shown in Table~\ref{tab:MRT}, \textbf{SCLCS}\textsubscript{Dense} achieves superior performance with lower variance across sampling ratios and evaluation depths, highlighting the benefit of density-aware sampling for capturing diverse patterns in group-comparison benchmarking. In this setting, the standard \textbf{SCLCS} (top-$k$) is less effective, suggesting that the best sampling strategy depends on the task’s structural demands. Finally, simpler heuristics such as k-Means and class-imbalance–prone criteria such as Entropy (Figure~\ref{balance}) perform less competitively, reinforcing the need for structure-aware selection.

The behavior of the Random baseline suggests that ranking-preserving selection differs from traditional single-model core-set objectives, and that methods designed for the latter may not transfer well.
The non-monotonic trend (30\% < 10\%) observed for \textbf{SCLCS} and \textbf{SCLCS}\textsubscript{Dense} is consistent with \textbf{Theorem~\ref{thm:unbalanced_topk}}, indicating that naïve score-based top-$k$ sampling can be brittle: it may over-represent dense clusters of typical patterns while missing rarer but important structures. This motivates \textbf{SCLCS}\textsubscript{Dense}, which mitigates this failure mode via density-aware sampling.

\paragraph{Empirical Validation of Theorem~\ref{thm:universal}}

\begin{table}[h]
\centering
\caption{Empirical validation of \textbf{Theorem 2}. Our modified Transformer approximates a diverse set of SPIs. The model demonstrates effective fitting and generalization.}
\label{tab:theorem2_validation}
\resizebox{\textwidth}{!}{%
\begin{tabular}{l|c|c|c||l|c|c|c}
\toprule
\textbf{SPI Operator} & \textbf{Train MSE (Start)} & \textbf{Train MSE (End)} & \textbf{Test MSE} & \textbf{SPI Operator} & \textbf{Train MSE (Start)} & \textbf{Train MSE (End)} & \textbf{Test MSE} \\
\midrule
pec\_orth & 0.0605 & 0.0584 & 0.0584 & plv\_multitaper\_mean & 0.3122 & 0.0588 & 0.0585 \\
phase\_multitaper\_mean & 0.0295 & 0.0264 & 0.0265 & cohmag\_multitaper\_max & 0.5721 & 0.0055 & 0.0057 \\
pli\_multitaper\_mean & 0.0239 & 0.0221 & 0.0222 & te\_kernel & 0.8513 & 0.0181 & 0.0182 \\
wpli\_multitaper\_mean & 0.0237 & 0.0221 & 0.0222 & bary\_euclidean\_max & 1.9840 & 0.1164 & 0.1153 \\
psi\_wavelet\_max & 6.0959 & 3.5522 & 3.5789 & xme\_gaussian & 0.9223 & 0.0155 & 0.0158 \\
ppc\_multitaper\_mean & 0.2082 & 0.1391 & 0.1382 & je\_kernel & 7.5831 & 0.1021 & 0.1025 \\
gwtau & 6.4630 & 1.8087 & 1.8124 & ce\_kernel & 0.9243 & 0.0141 & 0.0144 \\
icoh\_multitaper\_mean & 0.0797 & 0.0206 & 0.0206 & lcss\_constraint & 0.2608 & 0.0020 & 0.0020 \\
\bottomrule
\end{tabular}%
}
\end{table}

To empirically test the approximation capacity implied by \textbf{Theorem~\ref{thm:universal}}, we train our modified Transformer to approximate the FC matrices produced by 16 representative SPI operators selected from the taxonomy of~\citet{cliff2023unifying}. For each target SPI, we train a separate model on fMRI time series by minimizing the mean squared error (MSE) to the SPI-generated FC matrices. Table~\ref{tab:theorem2_validation} shows low final test MSE across all 16 targets, indicating that the model can closely approximate a diverse set of SPIs. Approximation fidelity varies by SPI, but the overall trend supports \textbf{Theorem~\ref{thm:universal}}.

Together, \textbf{Theorem~\ref{thm:universal}} and Table~\ref{tab:theorem2_validation} suggest that the architecture is expressive enough to serve as a structural probe for fMRI time series. Small approximation errors (e.g., imperfect emulation of discrete statistical tests) need not invalidate our stability signal: \textbf{SPS} aims to distinguish structurally stable samples (low \textbf{SPS}) from unstable ones (high \textbf{SPS}), rather than to maximize approximation fidelity. Quantifying how approximation error propagates to ranking preservation is an important direction for future work. Full details and convergence curves are provided in \textbf{Appendix~\ref{convuni}}.

\paragraph{Sample Coverage Balance Analysis}
\begin{figure}[!h]
  \centering
  \includegraphics[width=\linewidth]{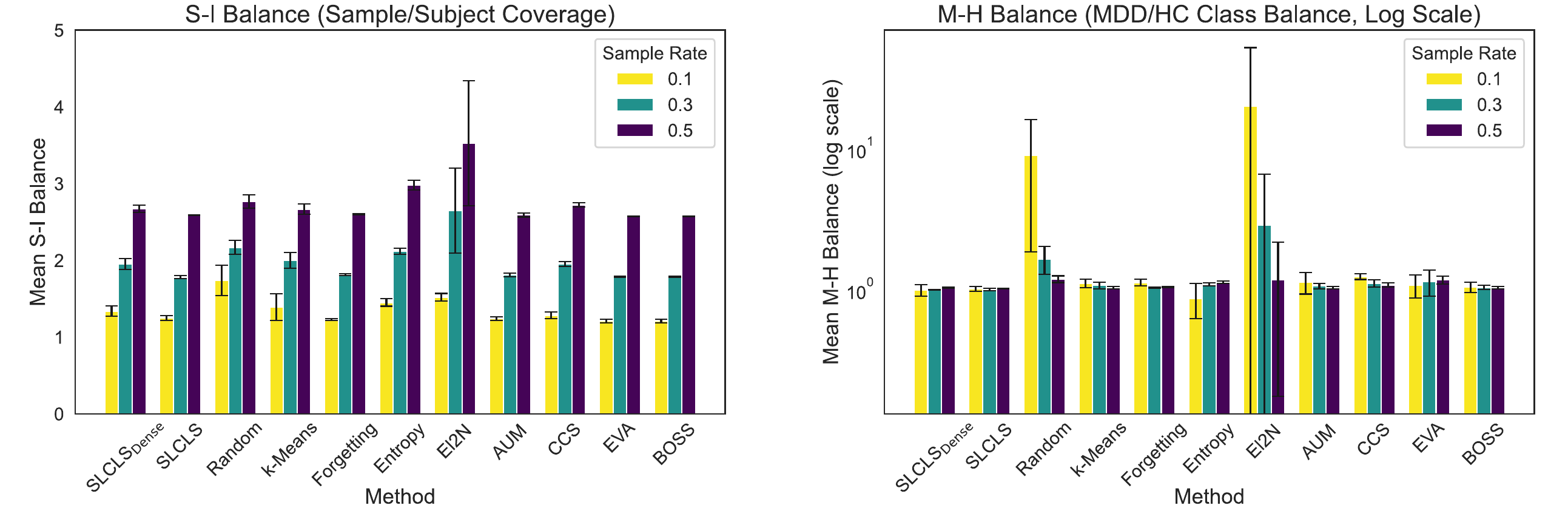}
  \caption{Sample coverage balance on subjects and MDD/HC of baselines.}
  \label{balance}
\end{figure}

Beyond performance, we assess the reliability and representativeness of selection by analyzing sample coverage balance (Figure~\ref{balance}). We introduce two metrics: \emph{S-I Balance} (selected samples per subject, lower indicates broader subject coverage) and \emph{M-H Balance} (MDD-to-HC ratio; deviation from 1 indicates class imbalance). As shown in Figure~\ref{balance}, \textbf{SCLCS} and \textbf{SCLCS\textsubscript{Dense}} maintain balanced coverage with low variance on both metrics, whereas several baselines are unstable. Entropy is particularly sensitive to class-label effects, selecting almost exclusively MDD subjects at low ratios. This produces a skewed, unrepresentative core-set and can make downstream benchmarking misleading, highlighting the risk of naïve score-based selection.

\subsection{Qualitative Results: Visualizing SPS Dynamics}
\label{sec:qualitative_analysis}

To provide an intuitive check of the \textbf{SPS} metric, we visualize the evolution of the learned attention map
$\mathbf{A}_{(\mathbf{X})}^e$ over the early training epochs $e$.
While the results in \textbf{Appendix~\ref{app:theorem_exp}} indicate that attention maps typically stabilize after $\sim 50$ epochs,
Figure~\ref{fig:sps_dynamics} shows that the stabilization dynamics vary across samples:
a low-\textbf{SPS} sample (top row) rapidly converges to a stable structural pattern,
whereas a high-\textbf{SPS} sample (bottom row) exhibits sustained fluctuations.
These observations suggest that \textbf{SPS} reflects a stable, sample-specific property rather than transient optimization noise,
supporting its use for identifying foundational samples for benchmarking.

\begin{figure}[h!]
    \centering
    \includegraphics[width=\linewidth]{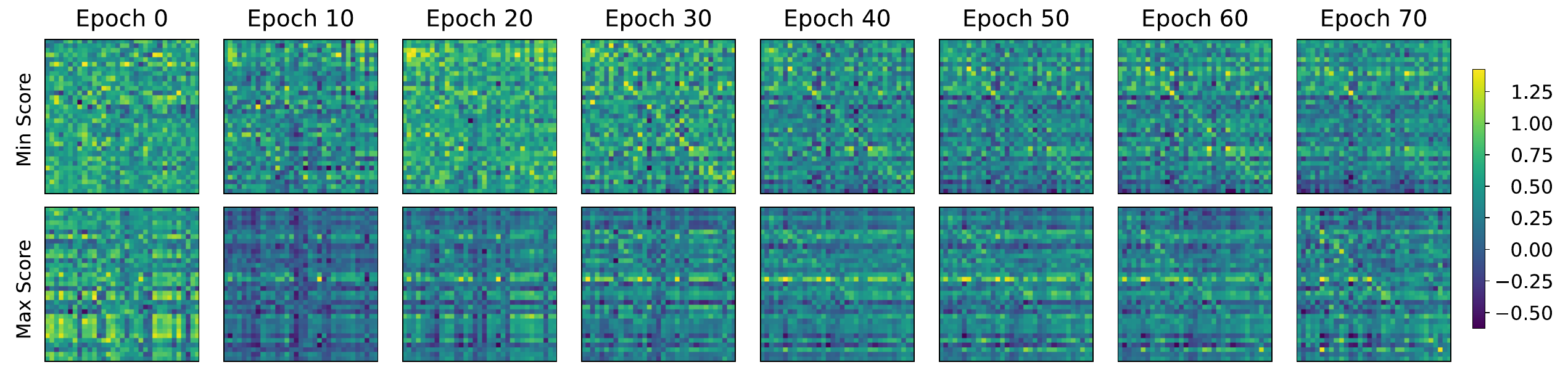} 
    \caption{The evolution of the learned attention map $\textbf{A}_{(\textbf{X})}^e$ across training epochs.}
    \label{fig:sps_dynamics}
\end{figure}

\section{Conclusion}

In this work, we address the computational bottleneck of large-scale FC operator (SPI) benchmarking by casting core-set selection as a ranking-preservation task. Our key technical contributions are: (1) A modified Transformer architecture with a universal approximation guarantee for continuous SPI mappings under our assumptions. (2) The \textbf{SPS} metric to identify structurally stable samples. (3) The \textbf{SCLCS} framework, which outperforms $9$ baselines in ranking-preservation evaluation. By accelerating FC benchmarking, \textbf{SCLCS} makes large-scale, pre-analysis SPI comparisons practical and supports more reproducible computational neuroscience.

\section*{Acknowledgements}
This work is supported by Natural Science Foundation of China (No. 62402398, No. 72374173), the Fundamental Research Funds for the Central Universities (No. SWU-XDJH202303, No. SWU-KR24025), the Technological Innovation and Application Development Project of Chongqing (No. CSTB2025TIAD-KPX0027), University Innovation Research Group of Chongqing (No. CXQT21005), China Scholarship Council (CSC) program (No. 202406990056) and Chongqing Graduate Research Innovation Project (CYB240088). The experiments are supported by the High Performance Computing clusters and the Large-Scale Instrument Sharing Platform (H20 GPU Server) at Southwest University.

\section*{Reproducibility Statement}
\label{repro_state}
To ensure our research is fully reproducible, we have made our code, data sources, and experimental details available as follows:

\begin{itemize}
    \item \textbf{Code:} The complete source code for our SCLCS framework and all experiments is publicly available on \url{https://github.com/lzhan94swu/SCLCS}.

    \item \textbf{Dataset and Preprocessing:} The REST-meta-MDD~\citep{yan2019reduced} dataset is publicly accessible at \url{https://rfmri.org/REST-meta-MDD}. Our detailed data preprocessing pipeline is described in \textbf{Appendix~\ref{app:dp}}.

    \item \textbf{External Libraries:} Our analysis relies on the \texttt{pyspi} library~\citep{cliff2023unifying}, which is publicly available at \url{https://github.com/DynamicsAndNeuralSystems/pyspi}. The specific criteria used to select SPIs for our benchmark are detailed in \textbf{Appendix~\ref{app:comps}}.

    \item \textbf{Experimental Settings:} A summary of the experimental setup is presented in \textbf{Section~\ref{settings}}, with a comprehensive breakdown of all parameters and configurations available in \textbf{Appendix~\ref{app:repro}}.

    \item \textbf{Theoretical Proofs:} Complete proofs for all theorems, propositions, and lemmas presented in this paper can be found in \textbf{Appendix~\ref{app:interference}–\ref{app:discrep}}.
\end{itemize}

\section*{Ethics Statement}
Our work provides a framework for evaluating and selecting computational models used in neuroimaging analysis, which can be applied to clinical tasks such as disease diagnosis. Therefore, it is important to consider the potential societal impacts of the models ultimately chosen via our benchmarking process. A potential negative impact could arise if a core-set, while efficient, is not perfectly representative of the full dataset's diversity, leading to the selection of a model that is biased or performs suboptimally on underrepresented demographic or clinical groups. To mitigate this, we emphasize that our framework is a tool for pre-clinical scientific validation. Any model selected using our approach for real-world medical scenarios must undergo its own rigorous, independent clinical validation, and the final diagnostic decision must always remain with a qualified physician.

\bibliography{iclr2026_conference}
\bibliographystyle{iclr2026_conference}

\appendix
\renewcommand{\thetable}{A\arabic{table}} 
\renewcommand{\thefigure}{A\arabic{figure}} 
\setcounter{table}{0}
\setcounter{figure}{0}
\newtheorem*{theorem*}{Theorem}

\section{Proof of Theorem~\ref{thm:interference}}
\label{app:interference}

\begin{theorem*}[Interference of Averaged Attention, full version]
Let $\{\mathbf{A}_{h}\}_{h=1}^{H}$ be a collection of row stochastic\footnote{%
Each row is a probability distribution produced by a softmax:
$\mathbf{A}_{h}(i,j)\ge 0$ and $\sum_{j=1}^{N}\mathbf{A}_{h}(i,j)=1$.}
attention matrices with $\mathbf{A}_{h}\in\mathbb{R}^{N\times N}$.
For every row index $i\in\{1,\dots,N\}$ assume there exist sets
$S_{h}^{(i)} \subseteq \{1,\dots,N\}$ that are pairwise disjoint, meaning
$S_{h}^{(i)}\cap S_{h'}^{(i)}=\varnothing$ for all $h\neq h'$,
and such that
\[
    \mathbf{A}_{h}(i,j)=0 \quad \text{for all } j\notin S_{h}^{(i)}.
\]
Define the uniform average $\bar{\mathbf{A}}:=\tfrac1H\sum_{h=1}^{H}\mathbf{A}_{h}$.
Then for every row $i$:
\begin{enumerate}[label=(\alph*)]
    \item \textbf{Support expansion.}
    $\operatorname{supp}(\bar{\mathbf{a}}^{(i)})=\bigcup_{h=1}^H S_h^{(i)}$.
    If $H\ge2$, then for every $h$,
    $\operatorname{supp}(\bar{\mathbf{a}}^{(i)})\not\subseteq S_h^{(i)}$.
    \item \textbf{Entropy inflation.}
    With $\mathcal{H}(\mathbf{p}) := -\sum_{j=1}^{N} p_{j}\log p_{j}$,
    \[
        \mathcal{H}\!\bigl(\bar{\mathbf{a}}^{(i)}\bigr)
        \;\ge\;
        \frac{1}{H}\sum_{h=1}^{H}\mathcal{H}\!\bigl(\mathbf{a}_{h}^{(i)}\bigr)
        \;\ge\;
        \min_{1\le h\le H}\mathcal{H}\!\bigl(\mathbf{a}_{h}^{(i)}\bigr),
    \]
    and strict inequality holds whenever the vectors
    $\{\mathbf{a}_{h}^{(i)}\}_{h=1}^{H}$ are not all identical.
    In particular, under the disjoint mask assumption, strict inequality holds for every $i$ whenever $H\ge2$.
\end{enumerate}
\end{theorem*}

\begin{proof}
Fix a row index $i$ and write $\mathbf{p}_h := \mathbf{a}_{h}^{(i)}\in\Delta^{N-1}$.
By definition,
\[
\bar{\mathbf{p}}
:= \bar{\mathbf{a}}^{(i)}
= \frac{1}{H}\sum_{h=1}^{H}\mathbf{p}_h,
\qquad
\bar{\mathbf{p}}_j
= \frac{1}{H}\sum_{h=1}^{H}(\mathbf{p}_h)_j.
\]

\textbf{(a) Support expansion.}
All entries are nonnegative, so $\bar{\mathbf{p}}_j>0$ if and only if there exists a head $h$
with $(\mathbf{p}_h)_j>0$.
Under the masking assumption, $(\mathbf{p}_h)_j>0$ implies $j\in S_h^{(i)}$.
Conversely, since each $\mathbf{p}_h$ is a probability vector supported on $S_h^{(i)}$,
every $j\in S_h^{(i)}$ with nonzero mass in head $h$ contributes a positive term to $\bar{\mathbf{p}}_j$.
Therefore,
\[
\operatorname{supp}(\bar{\mathbf{p}})
= \bigcup_{h=1}^{H}\operatorname{supp}(\mathbf{p}_h)
\subseteq \bigcup_{h=1}^{H} S_h^{(i)}.
\]
Moreover, because $\mathbf{p}_h$ is row stochastic and supported inside $S_h^{(i)}$,
we have $\operatorname{supp}(\mathbf{p}_h)\neq\varnothing$ and $\operatorname{supp}(\mathbf{p}_h)\subseteq S_h^{(i)}$,
so $\operatorname{supp}(\bar{\mathbf{p}})$ equals the union of the head supports and is contained in the union of the masks.
If, as in the theorem statement, the intended structural supports are exactly the mask sets
(that is, the mask defines which indices can receive positive mass),
then $\operatorname{supp}(\mathbf{p}_h)=S_h^{(i)}$ and hence
$\operatorname{supp}(\bar{\mathbf{p}})=\bigcup_{h=1}^{H} S_h^{(i)}$.
When $H\ge2$ and the sets $\{S_h^{(i)}\}_{h=1}^{H}$ are pairwise disjoint, the union strictly contains each $S_h^{(i)}$,
so $\operatorname{supp}(\bar{\mathbf{p}})\not\subseteq S_h^{(i)}$ for every $h$.

\textbf{(b) Entropy inflation.}
The Shannon entropy $\mathcal H$ is strictly concave on the probability simplex.
By Jensen's inequality,
\[
\mathcal{H}(\bar{\mathbf{p}})
\;\ge\;
\frac{1}{H}\sum_{h=1}^{H}\mathcal{H}(\mathbf{p}_h)
\;\ge\;
\min_{1\le h\le H}\mathcal{H}(\mathbf{p}_h),
\]
and the first inequality is strict unless $\mathbf{p}_1=\cdots=\mathbf{p}_H$.
Under the disjoint mask assumption with $H\ge2$, the vectors cannot all be identical:
if $\mathbf{p}_h=\mathbf{p}_{h'}$ for some $h\neq h'$, then their supports coincide and are nonempty,
so $S_h^{(i)}\cap S_{h'}^{(i)}\neq\varnothing$, contradicting disjointness.
Hence $\mathbf{p}_1,\dots,\mathbf{p}_H$ are not all identical, so
$\mathcal{H}(\bar{\mathbf{p}}) > \frac{1}{H}\sum_{h=1}^{H}\mathcal{H}(\mathbf{p}_h)$
and therefore $\mathcal{H}(\bar{\mathbf{p}}) > \min_h \mathcal{H}(\mathbf{p}_h)$.

Together, (a) and (b) show that uniform averaging introduces additional nonzero entries
and increases entropy, which blurs head specific structural patterns.
\end{proof}

\section{Proof of Theorem~\ref{thm:universal}}
\label{app:universal}

\begin{proof}
We use an existing attention-only universal approximation result as the main engine
and then specialize it to row-stochastic matrix-valued targets.

\paragraph{Step 1 (Reduce to sequence-to-sequence approximation).}
View $\mathbf{X}\in\mathbb{R}^{N\times T}$ as a length-$N$ sequence with token dimension $T$.
Likewise, view $S(\mathbf{X})\in\mathbb{R}^{N\times N}$ as a length-$N$ sequence of
$row$ vectors in $\mathbb{R}^{N}$.
Since $\mathcal{X}$ is compact and $S$ is continuous, this is a continuous
sequence-to-sequence map on a compact domain.

\paragraph{Step 2 (Attention-only universality on compact sets).}
Recent results show that (softmax-)attention-only architectures with linear projections
are universal approximators for continuous sequence-to-sequence maps on compact domains~\citep{hu2025universal}.
In particular, attention modules can simulate piecewise-linear bases and hence achieve uniform approximation without requiring feed-forward sublayers.
We invoke such a theorem.

\paragraph{Step 3 (Constrain the output to be row-stochastic).}
Our target operator satisfies $S(\mathbf{X})\in\Delta^{N-1\times N}$ for all $\mathbf{X}$.
The model family in \eqref{eq:universal_family} is also row-stochastic by construction:
each head output is row-stochastic (row-wise softmax), and the convex mixture
with $\boldsymbol{\alpha}\in\Delta^{H-1}$ preserves row-stochasticity.
Therefore the approximating attention-only construction can be chosen to lie entirely
inside $\Delta^{N-1\times N}$.

\paragraph{Step 4 (Uniform approximation in Frobenius norm).}
The cited attention-only universality provides uniform approximation in a sup norm over
the compact domain. Since $\|\cdot\|_F \le \sqrt{N}\|\cdot\|_{\infty,\infty}$,
the same parameter choice yields
\[
\sup_{\mathbf{X}\in\mathcal{X}}
\bigl\|\mathbf{A}_\theta(\mathbf{X})-S(\mathbf{X})\bigr\|_{F}
<\varepsilon
\]
after tightening constants.

This completes the proof.
\end{proof}

\section{Proof of Proposition~\ref{prop:mixture_sps}}
\label{app:mixture_sps}

\begin{proof}
Because $Z_e$ and $Z_{e-1}$ are i.i.d.,
\begin{align}
\mathbb{E}[\Delta_e]
&=\sum_{k=1}^{K}\sum_{l=1}^{K}\Pr[Z_e=S^{(k)},\,Z_{e-1}=S^{(l)}]\,
\|S^{(k)}-S^{(l)}\|_F^2 \notag\\
&=\sum_{k=1}^{K}\sum_{l=1}^{K}\lambda_k\lambda_l D_{kl}\notag,
\end{align}
which gives the first equality in \eqref{eq:mixture_exact}.
Since $D_{kk}=0$ and $D_{kl}=D_{lk}$,
\[
\sum_{k=1}^{K}\sum_{l=1}^{K}\lambda_k\lambda_l D_{kl}
=
2\sum_{k<l}\lambda_k\lambda_l D_{kl},
\]
proving the second equality in \eqref{eq:mixture_exact}.

For the bounds \eqref{eq:mixture_bounds}, note that for all $k<l$,
$D_{\min}\le D_{kl}\le D_{\max}$, hence
\[
2D_{\min}\sum_{k<l}\lambda_k\lambda_l
\;\le\;
2\sum_{k<l}\lambda_k\lambda_l D_{kl}
\;\le\;
2D_{\max}\sum_{k<l}\lambda_k\lambda_l.
\]
Finally,
\[
2\sum_{k<l}\lambda_k\lambda_l
=
\Bigl(\sum_{k=1}^{K}\lambda_k\Bigr)^2-\sum_{k=1}^{K}\lambda_k^2
=
1-\sum_{k=1}^{K}\lambda_k^2,
\]
which is the Gini impurity of $\{\lambda_k\}$.
This yields \eqref{eq:mixture_bounds}.

If $D_{kl}\equiv D$ for all $k\neq l$, then \eqref{eq:mixture_exact} becomes
\[
\mathbb{E}[\Delta_e]
=
2D\sum_{k<l}\lambda_k\lambda_l
=
D\Bigl(1-\sum_{k=1}^{K}\lambda_k^2\Bigr),
\]
establishing \eqref{eq:mixture_isotropic}.
\end{proof}

\section{Lemma~\ref{lem:sps_consistency} and Proof}\label{app:consist}
\begin{lemma}[Consistency of SPS]
\label{lem:sps_consistency}
Let $\{\mathbf{A}^{(e)}_{(\mathbf{X})}\}_{e\ge 0}$ be the sequence of
attention‐based structure matrices for a fixed sample $\mathbf{X}$ generated by a stochastic optimization algorithm.
Assume the sequence of differences
\begin{equation}
\Delta_e(\mathbf{X}) \;=\;
\bigl\|\mathbf{A}^{(e)}_{(\mathbf{X})}-\mathbf{A}^{(e-1)}_{(\mathbf{X})}\bigr\|_F^2
\end{equation}
forms a {\em stationary and ergodic} process with finite mean
$\sigma^{2}(\mathbf{X})=\mathbb{E}[\Delta_e(\mathbf{X})]$.
Then the empirical SPS estimator
\begin{equation}
\widehat{\mathrm{SPS}}_L(\mathbf{X}) \;=\;
\frac{1}{L}\sum_{e=1}^{L}\Delta_e(\mathbf{X})
\end{equation}
converges almost surely to~$\sigma^{2}(\mathbf{X})$ as $L\to\infty$:
\(
\widehat{\mathrm{SPS}}_L(\mathbf{X}) \;
\longrightarrow\;
\sigma^{2}(\mathbf{X})
\quad\text{a.s.}
\)
\end{lemma}

\begin{proof}
Because $\{\Delta_e(\mathbf{X})\}$ is assumed stationary and ergodic with finite first moment,
Birkhoff’s pointwise ergodic theorem~\citep{birkhoff1931proof} applies:
\begin{equation}
\frac1L\sum_{e=1}^{L}\Delta_e(\mathbf{X})
\;\xrightarrow{\text{a.s.}}\;
\mathbb{E}[\Delta_e(\mathbf{X})]
\;=\;\sigma^{2}(\mathbf{X}).
\end{equation}
But the left–hand side is precisely $\widehat{\mathrm{SPS}}_L(\mathbf{X})$.
Hence the estimator is strongly consistent.
\end{proof}

\section{Proof of Theorem~\ref{thm:unbalanced_topk}}
\label{app:proof_unbalanced_topk}

\begin{proof}
Write the scores as $\{s_i\}_{i=1}^{N}$ and let $s_{(k)}$ be the $k$-th order statistic.
Equivalently, $S_k=\{\mathbf{x}_i: s_i \le s_{(k)}\}$ up to tie-breaking on a null event
(because the score distributions are continuous).

\paragraph{Step 1: Identify the population selection threshold.}
Let $\tau$ satisfy \eqref{eq:tau_def}.
Intuitively, $\tau$ is the population score threshold whose expected accepted fraction is $\rho$:
\[
\mathbb{E}\Bigl[\frac{1}{N}\sum_{i=1}^{N}\mathbbm{1}\{s_i\le \tau\}\Bigr]
=\pi_p F_p(\tau)+\pi_q F_q(\tau)=\rho.
\]

\paragraph{Step 2: The empirical threshold concentrates.}
Define the empirical accepted fraction at threshold $t$:
\[
G_N(t):=\frac{1}{N}\sum_{i=1}^{N}\mathbbm{1}\{s_i\le t\}.
\]
Because scores are independent across samples and each indicator is bounded,
a standard concentration argument (e.g., Hoeffding) yields that $G_N(t)$
concentrates uniformly on compact intervals around its mean
$G(t):=\pi_p F_p(t)+\pi_q F_q(t)$.
Since $F_p,F_q$ are continuous and $G$ is strictly increasing at $\tau$ (under mild regularity),
it follows that the empirical quantile $s_{(k)}$ converges in probability to $\tau$:
\begin{equation}
\label{eq:quantile_conv}
s_{(k)} \xrightarrow[N\to\infty]{\Pr} \tau.
\end{equation}

\paragraph{Step 3: Selected cluster counts converge to their expectations.}
Conditional on $s_{(k)}$, the number of selected points from cluster $C_p$ is
\[
|S_k\cap C_p|=\sum_{\mathbf{x}\in C_p}\mathbbm{1}\{s(\mathbf{x})\le s_{(k)}\}.
\]
Given $s_{(k)}$, the indicators in the sum are i.i.d.\ Bernoulli with parameter $F_p(s_{(k)})$,
so by the law of large numbers (and Slutsky using \eqref{eq:quantile_conv}),
\begin{equation}
\label{eq:count_conv_p}
\frac{|S_k\cap C_p|}{n_p}\xrightarrow{\Pr} F_p(\tau).
\end{equation}
Similarly,
\begin{equation}
\label{eq:count_conv_q}
\frac{|S_k\cap C_q|}{n_q}\xrightarrow{\Pr} F_q(\tau).
\end{equation}

\paragraph{Step 4: Convert to selected proportions and compute the bias.}
Divide \eqref{eq:count_conv_p} by $k/N\to \rho$:
\[
\widehat{\pi}_p(S_k)
=\frac{|S_k\cap C_p|}{k}
=
\frac{|S_k\cap C_p|/N}{k/N}
=
\frac{(n_p/N)\cdot (|S_k\cap C_p|/n_p)}{k/N}
\xrightarrow{\Pr}
\frac{\pi_p F_p(\tau)}{\rho}.
\]
Using \eqref{eq:tau_def}, we compute
\[
\frac{\pi_p F_p(\tau)}{\rho}-\pi_p
=\pi_p\Bigl(\frac{F_p(\tau)-\rho}{\rho}\Bigr)
=\pi_p\Bigl(\frac{F_p(\tau)-\pi_pF_p(\tau)-\pi_qF_q(\tau)}{\rho}\Bigr)
=\frac{\pi_p\pi_q}{\rho}\bigl(F_p(\tau)-F_q(\tau)\bigr).
\]
Under \eqref{eq:cdf_gap}, this equals $\delta:=\frac{\pi_p\pi_q}{\rho}\gamma>0$,
establishing the first limit in \eqref{eq:limit_bias}.
The statement for $\widehat{\pi}_q(S_k)$ follows since $\widehat{\pi}_q(S_k)=1-\widehat{\pi}_p(S_k)$,
and then $\Delta_k=2|\widehat{\pi}_p(S_k)-\pi_p|\xrightarrow{\Pr}2\delta>0$.
\end{proof}

\section{Theorem~\ref{thm:eps_cover} and Proof}\label{app:cover}
\begin{theorem}[$\varepsilon$-coverage of density-reweighted sampling]
\label{thm:eps_cover}
Fix \(0<\delta<1\) and \(\varepsilon>0\).
Let \(\tilde{\mathcal X}_c\) be the candidate pool with \(n=\lvert\tilde{\mathcal X}_c\rvert\),
and let \(S\subset\tilde{\mathcal X}_c\) be the subset of size \(m\) returned by the proposed sampling procedure.
Let \(N_{\varepsilon}\) be the \(\varepsilon\)-covering number of \(\mathcal X_c\) under
\(d(\mathbf A,\mathbf A')=\|\mathbf A-\mathbf A'\|_F\).
Under \textbf{\textup{Assumption~\ref{asmp:density}}}, if
\begin{equation}
m \;\ge\; \frac{n(\rho_{\max}+\tau)}{\rho_{\min}+\tau}
\Bigl(\log N_{\varepsilon}+\log(1/\delta)\Bigr),
\end{equation}
then \(S\) is an \(\varepsilon\)-cover of \(\mathcal X_c\) with probability at least \(1-\delta\).
\end{theorem}

Define the structure-representation space
\((\mathcal{M},d)\) with metric
\(d(\mathbf{A},\mathbf{A}')=\lVert \mathbf{A}-\mathbf{A}'\rVert_F\).
For \(\varepsilon>0\) let \(N_{\varepsilon}\) be the covering number of
\(\mathcal{X}_c\subset\mathcal{M}\). That is, the smallest number of closed
\(d\)-balls of radius \(\varepsilon\) needed to cover all subjects in
\(\mathcal{X}_c\).

\begin{assumption}
\label{asmp:density}
After the $\beta$-filter step, define the KDE-induced density
\(\rho(\mathbf{X}) := \hat{p}_{\mathrm{SPS}}\!\bigl(\mathrm{SPS}(\mathbf{X})\bigr)\),
where \(\hat{p}_{\mathrm{SPS}}\) is a Gaussian KDE fit on
\(\{\mathrm{SPS}(\mathbf{X}) : \mathbf{X}\in\tilde{\mathcal{X}}_c\}\).
Assume the estimator is bounded on the candidate pool:
\(0<\rho_{\min}\le\rho(\mathbf{X})\le\rho_{\max}<\infty\) for every
\(\mathbf{X}\in\tilde{\mathcal{X}}_c\).
\end{assumption}

\textbf{Assumption~\ref{asmp:density}} is mild because Gaussian KDE with finite
bandwidth produces a bounded, strictly positive estimate on any finite
sample.

\begin{proof}
Define the (normalized) sampling weight for \(\mathbf X\in\tilde{\mathcal X}_c\) as
\begin{equation}
w(\mathbf{X})
=
\frac{\frac{1}{\rho(\mathbf{X})+\tau}}
     {\sum_{\mathbf{Z}\in\tilde{\mathcal{X}}_c}\frac{1}{\rho(\mathbf{Z})+\tau}}.
\end{equation}
By \textbf{Assumption~\ref{asmp:density}}, for all \(\mathbf X\),
\(
\frac{1}{\rho(\mathbf{X})+\tau}\ge \frac{1}{\rho_{\max}+\tau}
\)
and
\(
\sum_{\mathbf{Z}}\frac{1}{\rho(\mathbf{Z})+\tau}\le \frac{n}{\rho_{\min}+\tau}.
\)
Hence every point has weight bounded below by
\begin{equation}
w(\mathbf X)\;\ge\;
\frac{1/(\rho_{\max}+\tau)}{n/(\rho_{\min}+\tau)}
=
\frac{\rho_{\min}+\tau}{n(\rho_{\max}+\tau)}
\;=:\; w_{\min}.
\end{equation}

Let \(\{B_1,\dots,B_{N_{\varepsilon}}\}\) be a collection of closed \(d\)-balls of radius \(\varepsilon\)
covering \(\mathcal X_c\). Since \(\mathcal X_c\subseteq \tilde{\mathcal X}_c\),
each \(B_j\) contains at least one candidate point. Therefore its total sampling mass satisfies
\(
W_j := \sum_{\mathbf X\in B_j\cap \tilde{\mathcal X}_c} w(\mathbf X)\ge w_{\min}.
\)

Consider sequential sampling without replacement where at each draw we sample from the remaining points
proportionally to their weights (renormalized). If we have not yet sampled from \(B_j\), then removing points
outside \(B_j\) can only increase the renormalized mass of \(B_j\); thus, at every draw the conditional
probability of selecting a point outside \(B_j\) is at most \(1-W_j\le 1-w_{\min}\).
Therefore
\begin{equation}
\Pr[B_j\cap S=\varnothing] \;\le\; (1-w_{\min})^{m}
\;\le\; \exp(-w_{\min}m).
\end{equation}

Choose \(m\) so that \(\exp(-w_{\min}m)\le \delta/N_{\varepsilon}\), i.e.
\(m \ge \frac{1}{w_{\min}}(\log N_{\varepsilon}+\log(1/\delta))\).
A union bound over the \(N_{\varepsilon}\) balls yields
\(
\Pr[\exists j:\, B_j\cap S=\varnothing]\le \delta.
\)
Hence with probability at least \(1-\delta\), every ball contains at least one sampled point,
so \(S\) is an \(\varepsilon\)-cover of \(\mathcal X_c\).
\end{proof}

\section{Theorem~\ref{thm:discrep} and Proof}\label{app:discrep}

\begin{theorem}[Expectation discrepancy under $\varepsilon$-coverage]
\label{thm:discrep}
Let $(\mathcal M,d)$ be the structure-representation space with
$d(\mathbf A,\mathbf A')=\|\mathbf A-\mathbf A'\|_F$ and representation map
$\mathbf X\mapsto \mathbf A^{(\mathbf X)}\in\mathcal M$.
Let $P$ be a probability distribution supported on $\mathcal X$.
Assume $\tilde{\mathcal X}\subset\mathcal X$ is an $\varepsilon$-cover of $\mathcal X$ in $\mathcal M$:
for every $\mathbf X\in\mathcal X$ there exists $\tilde{\mathbf X}\in\tilde{\mathcal X}$ with
$d(\mathbf A^{(\mathbf X)},\mathbf A^{(\tilde{\mathbf X})})\le \varepsilon$.

Let $\pi:\mathcal X\to \tilde{\mathcal X}$ be any (measurable) selection satisfying
$d(\mathbf A^{(\mathbf X)},\mathbf A^{(\pi(\mathbf X))})\le \varepsilon$ for all $\mathbf X$,
and define the push-forward measure $P_\pi := P\circ \pi^{-1}$ on $\tilde{\mathcal X}$.
Then for any function $f:\mathcal X\to\mathbb R$ that is $L$-Lipschitz w.r.t.\ $d$,
\begin{equation}
\left|
\mathbb E_{\mathbf X\sim P}\!\bigl[f(\mathbf X)\bigr]
-
\mathbb E_{\mathbf X\sim P_\pi}\!\bigl[f(\mathbf X)\bigr]
\right|
\le L\,\varepsilon.
\end{equation}
\end{theorem}

\begin{proof}
Because $f$ is $L$-Lipschitz w.r.t.\ $d$, for every $\mathbf X\in\mathcal X$,
\[
\bigl|f(\mathbf X)-f(\pi(\mathbf X))\bigr|
\le
L\, d\!\bigl(\mathbf A^{(\mathbf X)},\mathbf A^{(\pi(\mathbf X))}\bigr)
\le
L\,\varepsilon.
\]
Taking expectation under $P$ and using Jensen/triangle inequality,
\[
\left|
\mathbb E_{\mathbf X\sim P}\!\bigl[f(\mathbf X)\bigr]
-
\mathbb E_{\mathbf X\sim P}\!\bigl[f(\pi(\mathbf X))\bigr]
\right|
\le
\mathbb E_{\mathbf X\sim P}\!\bigl[|f(\mathbf X)-f(\pi(\mathbf X))|\bigr]
\le
L\,\varepsilon.
\]
Finally, by definition of the push-forward measure $P_\pi$,
\[
\mathbb E_{\mathbf X\sim P}\!\bigl[f(\pi(\mathbf X))\bigr]
=
\mathbb E_{\mathbf X\sim P_\pi}\!\bigl[f(\mathbf X)\bigr].
\]
Combining completes the proof.
\end{proof}

\section{Computational Complexity Analysis}\label{app:comps}

\subsection{Computational Complexity and Selection of SPIs}
\begin{figure}[ht]
  \centering
  \includegraphics[width=0.65\linewidth]{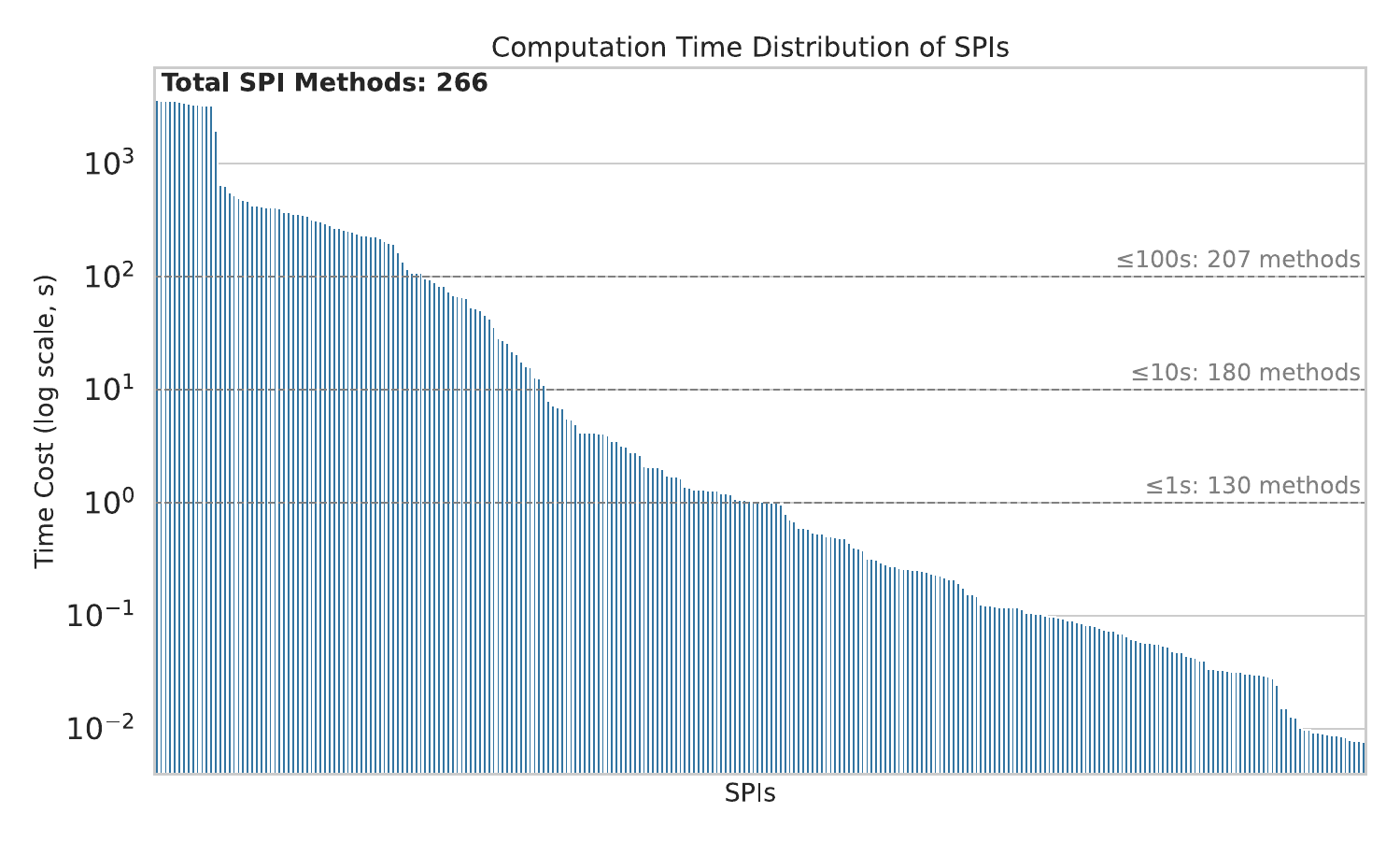}
  \caption{Time consumption of different SPIs on a single sample.}
  \label{time_comp}
\end{figure}

\begin{table}[ht]
\centering
\caption{SPIs included in \texttt{pyspi} (\textbf{bolded} ones are used in this paper).}\label{spis}
\renewcommand{\arraystretch}{1.2}
\setlength{\tabcolsep}{8pt}
\resizebox{\textwidth}{!}{%
\begin{tabular}{|l|l|l|l|}
\toprule
\textbf{cov\_EllipticEnvelope} & \textbf{cov\_GraphicalLasso} & \textbf{cov\_LedoitWolf} & \textbf{cov\_MinCovDet} \\
 \textbf{cov\_OAS} & \textbf{cov\_ShrunkCovariance} & \textbf{cov-sq\_EmpiricalCovariance} & \textbf{cov-sq\_EllipticEnvelope} \\
\textbf{cov-sq\_GraphicalLasso} & \textbf{cov-sq\_LedoitWolf} & \textbf{cov-sq\_MinCovDet} & \textbf{cov-sq\_OAS} \\
\textbf{cov\_PearsonCorrelation} &  \textbf{cov-sq\_ShrunkCovariance} & \textbf{prec\_EmpiricalCovariance} & \textbf{prec\_EllipticEnvelope} \\
\textbf{prec\_GraphicalLasso} & \textbf{prec\_LedoitWolf} & \textbf{prec\_MinCovDet} & \textbf{prec\_OAS} \\
\textbf{prec\_ShrunkCovariance} & \textbf{prec-sq\_EmpiricalCovariance} & \textbf{prec-sq\_EllipticEnvelope} & \textbf{prec-sq\_GraphicalLasso} \\
\textbf{prec-sq\_LedoitWolf} & \textbf{prec-sq\_MinCovDet} & \textbf{prec-sq\_OAS} & \textbf{prec-sq\_ShrunkCovariance}\\
 \textbf{kendalltau-sq} & \textbf{kendalltau} & \textbf{xcorr\_max\_sig-True} & \textbf{xcorr-sq\_max\_sig-True} \\
\textbf{xcorr\_mean\_sig-True} & \textbf{xcorr-sq\_mean\_sig-True} & \textbf{xcorr\_mean\_sig-False} & \textbf{xcorr-sq\_mean\_sig-False} \\
\textbf{pdist\_cityblock} &  \textbf{pdist\_cosine} & \textbf{pdist\_chebyshev} & \textbf{pdist\_canberra} \\
\textbf{pdist\_braycurtis} & \textbf{lcss\_constraint-sakoe-chiba} & \textbf{bary\_euclidean\_mean} & \textbf{bary\_euclidean\_max} \\
\textbf{bary-sq\_euclidean\_mean} & \textbf{bary-sq\_euclidean\_max} & \textbf{gwtau} & \textbf{je\_kernel\_W-0.5} \\
\textbf{ce\_gaussian} & \textbf{ce\_kernel\_W-0.5} & \textbf{xme\_gaussian\_k1} & \textbf{xme\_gaussian\_k10} \\
\textbf{mi\_gaussian} & \textbf{tlmi\_gaussian} & \textbf{te\_kernel\_W-0.25\_k-1} & \textbf{gc\_gaussian\_k-max-10\_tau-max-2} \\
\textbf{gc\_gaussian\_k-1\_kt-1\_l-1\_lt-1} & \textbf{te\_symbolic\_k-1\_kt-1\_l-1\_lt-1} & \textbf{te\_symbolic\_k-10\_kt-1\_l-1\_lt-1} & \textbf{phase\_multitaper\_mean\_fs-1\_fmin-0\_fmax-0-5} \\
\textbf{phase\_multitaper\_mean\_fs-1\_fmin-0\_fmax-0-25} & \textbf{phase\_multitaper\_mean\_fs-1\_fmin-0-25\_fmax-0-5} & \textbf{phase\_multitaper\_max\_fs-1\_fmin-0\_fmax-0-5} & \textbf{phase\_multitaper\_max\_fs-1\_fmin-0\_fmax-0-25} \\
\textbf{phase\_multitaper\_max\_fs-1\_fmin-0-25\_fmax-0-5} & \textbf{cohmag\_multitaper\_mean\_fs-1\_fmin-0\_fmax-0-5} & \textbf{cohmag\_multitaper\_mean\_fs-1\_fmin-0\_fmax-0-25} & \textbf{cohmag\_multitaper\_mean\_fs-1\_fmin-0-25\_fmax-0-5} \\
\textbf{cohmag\_multitaper\_max\_fs-1\_fmin-0\_fmax-0-5}& \textbf{cohmag\_multitaper\_max\_fs-1\_fmin-0\_fmax-0-25} & \textbf{cohmag\_multitaper\_max\_fs-1\_fmin-0-25\_fmax-0-5} & \textbf{icoh\_multitaper\_mean\_fs-1\_fmin-0\_fmax-0-5} \\
\textbf{icoh\_multitaper\_mean\_fs-1\_fmin-0\_fmax-0-25} & \textbf{icoh\_multitaper\_mean\_fs-1\_fmin-0-25\_fmax-0-5} & \textbf{icoh\_multitaper\_max\_fs-1\_fmin-0\_fmax-0-5} & \textbf{icoh\_multitaper\_max\_fs-1\_fmin-0\_fmax-0-25} \\
\textbf{icoh\_multitaper\_max\_fs-1\_fmin-0-25\_fmax-0-5} & \textbf{psi\_multitaper\_mean\_fs-1\_fmin-0\_fmax-0-5} & \textbf{psi\_multitaper\_mean\_fs-1\_fmin-0\_fmax-0-25} & \textbf{psi\_multitaper\_mean\_fs-1\_fmin-0-25\_fmax-0-5} \\
\textbf{plv\_multitaper\_mean\_fs-1\_fmin-0\_fmax-0-5} & \textbf{plv\_multitaper\_mean\_fs-1\_fmin-0\_fmax-0-25} & \textbf{plv\_multitaper\_mean\_fs-1\_fmin-0-25\_fmax-0-5} & \textbf{plv\_multitaper\_max\_fs-1\_fmin-0\_fmax-0-5} \\
\textbf{plv\_multitaper\_max\_fs-1\_fmin-0\_fmax-0-25} & \textbf{plv\_multitaper\_max\_fs-1\_fmin-0-25\_fmax-0-5} & \textbf{pli\_multitaper\_mean\_fs-1\_fmin-0\_fmax-0-5} & \textbf{pli\_multitaper\_mean\_fs-1\_fmin-0\_fmax-0-25} \\
\textbf{pli\_multitaper\_mean\_fs-1\_fmin-0-25\_fmax-0-5} & \textbf{pli\_multitaper\_max\_fs-1\_fmin-0\_fmax-0-5} & \textbf{pli\_multitaper\_max\_fs-1\_fmin-0\_fmax-0-25} & \textbf{pli\_multitaper\_max\_fs-1\_fmin-0-25\_fmax-0-5} \\
\textbf{wpli\_multitaper\_mean\_fs-1\_fmin-0\_fmax-0-5} & \textbf{wpli\_multitaper\_mean\_fs-1\_fmin-0\_fmax-0-25} & \textbf{wpli\_multitaper\_mean\_fs-1\_fmin-0-25\_fmax-0-5} & \textbf{wpli\_multitaper\_max\_fs-1\_fmin-0\_fmax-0-5} \\
\textbf{wpli\_multitaper\_max\_fs-1\_fmin-0\_fmax-0-25} & \textbf{wpli\_multitaper\_max\_fs-1\_fmin-0-25\_fmax-0-5} & \textbf{dspli\_multitaper\_mean\_fs-1\_fmin-0\_fmax-0-5} & \textbf{dspli\_multitaper\_mean\_fs-1\_fmin-0\_fmax-0-25} \\
\textbf{dspli\_multitaper\_mean\_fs-1\_fmin-0-25\_fmax-0-5} & \textbf{dspli\_multitaper\_max\_fs-1\_fmin-0\_fmax-0-5} & \textbf{dspli\_multitaper\_max\_fs-1\_fmin-0\_fmax-0-25} & \textbf{dspli\_multitaper\_max\_fs-1\_fmin-0-25\_fmax-0-5} \\
\textbf{dswpli\_multitaper\_mean\_fs-1\_fmin-0\_fmax-0-5} & \textbf{dswpli\_multitaper\_mean\_fs-1\_fmin-0\_fmax-0-25} & \textbf{dswpli\_multitaper\_mean\_fs-1\_fmin-0-25\_fmax-0-5} & \textbf{dswpli\_multitaper\_max\_fs-1\_fmin-0\_fmax-0-5} \\
\textbf{dswpli\_multitaper\_max\_fs-1\_fmin-0\_fmax-0-25} & \textbf{dswpli\_multitaper\_max\_fs-1\_fmin-0-25\_fmax-0-5} & \textbf{ppc\_multitaper\_mean\_fs-1\_fmin-0\_fmax-0-5} & \textbf{ppc\_multitaper\_mean\_fs-1\_fmin-0\_fmax-0-25} \\
\textbf{ppc\_multitaper\_mean\_fs-1\_fmin-0-25\_fmax-0-5} & \textbf{ppc\_multitaper\_max\_fs-1\_fmin-0\_fmax-0-5} & \textbf{ppc\_multitaper\_max\_fs-1\_fmin-0\_fmax-0-25} & \textbf{ppc\_multitaper\_max\_fs-1\_fmin-0-25\_fmax-0-5} \\
\textbf{gd\_multitaper\_delay\_fs-1\_fmin-0\_fmax-0-5} & \textbf{gd\_multitaper\_delay\_fs-1\_fmin-0-25\_fmax-0-5} & \textbf{sgc\_parametric\_mean\_fs-1\_fmin-0\_fmax-0-25\_order-1} & \textbf{sgc\_parametric\_max\_fs-1\_fmin-1e-05\_fmax-0-5\_order-1} \\
\textbf{psi\_wavelet\_mean\_fs-1\_fmin-0\_fmax-0-25\_mean} & \textbf{psi\_wavelet\_mean\_fs-1\_fmin-0-25\_fmax-0-5\_mean} & \textbf{psi\_wavelet\_max\_fs-1\_fmin-0\_fmax-0-5\_max} & \textbf{psi\_wavelet\_max\_fs-1\_fmin-0\_fmax-0-25\_max} \\
\textbf{psi\_wavelet\_max\_fs-1\_fmin-0-25\_fmax-0-5\_max} & \textbf{pec} & \textbf{pec\_orth} & \textbf{pec\_log} \\
\textbf{pec\_orth\_log} & \textbf{pec\_orth\_abs} & \textbf{pec\_orth\_log\_abs} & cov\_GraphicalLassoCV \\
cov-sq\_GraphicalLassoCV & prec\_GraphicalLassoCV & prec-sq\_GraphicalLassoCV & spearmanr-sq \\
spearmanr & pdist\_euclidean & dcorr & dcorr\_biased \\
mgc & hsic & hsic\_biased & hhg \\
mgcx\_maxlag-1 & mgcx\_maxlag-10 & dcorrx\_maxlag-1 & dcorrx\_maxlag-10 \\
dtw & dtw\_constraint-itakura & dtw\_constraint-sakoe-chiba & softdtw \\
softdtw\_constraint-itakura & softdtw\_constraint-sakoe-chiba & lcss & lcss\_constraint-itakura \\ bary\_dtw\_mean & bary\_dtw\_max & bary\_sgddtw\_mean & bary\_sgddtw\_max \\
bary\_softdtw\_mean & bary\_softdtw\_max & bary-sq\_dtw\_mean & bary-sq\_dtw\_max \\
bary-sq\_sgddtw\_mean & bary-sq\_sgddtw\_max & bary-sq\_softdtw\_mean & bary-sq\_softdtw\_max \\
anm & cds & reci & igci \\
ccm\_E-None\_mean & ccm\_E-None\_max & ccm\_E-None\_diff & ccm\_E-1\_mean \\
ccm\_E-1\_max & ccm\_E-1\_diff & ccm\_E-10\_mean & ccm\_E-10\_max \\
ccm\_E-10\_diff & je\_gaussian & je\_kozachenko & ce\_kozachenko \\
cce\_gaussian & cce\_kozachenko & cce\_kernel\_W-0.5 & xme\_kozachenko\_k1 \\
xme\_kernel\_W-0.5\_k1 & xme\_kozachenko\_k10 & xme\_kernel\_W-0.5\_k10 & di\_gaussian \\
di\_kozachenko & di\_kernel\_W-0.5 & si\_gaussian\_k-1 & si\_kozachenko\_k-1 \\
si\_kernel\_W-0.5\_k-1 & mi\_kraskov\_NN-4 & mi\_kraskov\_NN-4\_DCE & mi\_kernel\_W-0.25 \\
tlmi\_kraskov\_NN-4 & tlmi\_kraskov\_NN-4\_DCE & tlmi\_kernel\_W-0.25 & te\_kraskov\_NN-4\_k-max-10\_tau-max-4\\
te\_kraskov\_NN-4\_DCE\_k-max-10\_tau-max-4 & te\_kraskov\_NN-4\_DCE\_k-2\_kt-1\_l-1\_lt-1 & te\_kraskov\_NN-4\_DCE\_k-1\_kt-1\_l-1\_lt-1 & te\_kraskov\_NN-4\_k-1\_kt-1\_l-1\_lt-1 \\
phi\_star\_t-1\_norm-0 & phi\_star\_t-1\_norm-1 & phi\_Geo\_t-1\_norm-0 & phi\_Geo\_t-1\_norm-1 \\
dtf\_multitaper\_mean\_fs-1\_fmin-0\_fmax-0-5 & dtf\_multitaper\_mean\_fs-1\_fmin-0\_fmax-0-25 & dtf\_multitaper\_mean\_fs-1\_fmin-0-25\_fmax-0-5 & dtf\_multitaper\_max\_fs-1\_fmin-0\_fmax-0-5 \\
dtf\_multitaper\_max\_fs-1\_fmin-0\_fmax-0-25 & dtf\_multitaper\_max\_fs-1\_fmin-0-25\_fmax-0-5 & dcoh\_multitaper\_mean\_fs-1\_fmin-0\_fmax-0-5 & dcoh\_multitaper\_mean\_fs-1\_fmin-0\_fmax-0-25 \\
dcoh\_multitaper\_mean\_fs-1\_fmin-0-25\_fmax-0-5 & dcoh\_multitaper\_max\_fs-1\_fmin-0\_fmax-0-5 & dcoh\_multitaper\_max\_fs-1\_fmin-0\_fmax-0-25 & dcoh\_multitaper\_max\_fs-1\_fmin-0-25\_fmax-0-5 \\
pdcoh\_multitaper\_mean\_fs-1\_fmin-0\_fmax-0-5 & pdcoh\_multitaper\_mean\_fs-1\_fmin-0\_fmax-0-25 & pdcoh\_multitaper\_mean\_fs-1\_fmin-0-25\_fmax-0-5 & pdcoh\_multitaper\_max\_fs-1\_fmin-0\_fmax-0-5 \\
pdcoh\_multitaper\_max\_fs-1\_fmin-0\_fmax-0-25 & pdcoh\_multitaper\_max\_fs-1\_fmin-0-25\_fmax-0-5 & gpdcoh\_multitaper\_mean\_fs-1\_fmin-0\_fmax-0-5 & gpdcoh\_multitaper\_mean\_fs-1\_fmin-0\_fmax-0-25 \\
gpdcoh\_multitaper\_mean\_fs-1\_fmin-0-25\_fmax-0-5 & gpdcoh\_multitaper\_max\_fs-1\_fmin-0\_fmax-0-5 & gpdcoh\_multitaper\_max\_fs-1\_fmin-0\_fmax-0-25 & gpdcoh\_multitaper\_max\_fs-1\_fmin-0-25\_fmax-0-5 \\
ddtf\_multitaper\_mean\_fs-1\_fmin-0\_fmax-0-5 & ddtf\_multitaper\_mean\_fs-1\_fmin-0\_fmax-0-25 & ddtf\_multitaper\_mean\_fs-1\_fmin-0-25\_fmax-0-5 & ddtf\_multitaper\_max\_fs-1\_fmin-0\_fmax-0-5 \\
ddtf\_multitaper\_max\_fs-1\_fmin-0\_fmax-0-25 & ddtf\_multitaper\_max\_fs-1\_fmin-0-25\_fmax-0-5 & gd\_multitaper\_delay\_fs-1\_fmin-0\_fmax-0-25 & sgc\_nonparametric\_mean\_fs-1\_fmin-0\_fmax-0-5 \\
sgc\_nonparametric\_mean\_fs-1\_fmin-0\_fmax-0-25 & sgc\_nonparametric\_mean\_fs-1\_fmin-0-25\_fmax-0-5 & sgc\_nonparametric\_max\_fs-1\_fmin-0\_fmax-0-5 & sgc\_nonparametric\_max\_fs-1\_fmin-0\_fmax-0-25 \\
sgc\_nonparametric\_max\_fs-1\_fmin-0-25\_fmax-0-5 & sgc\_parametric\_mean\_fs-1\_fmin-0\_fmax-0-5\_order-None & sgc\_parametric\_mean\_fs-1\_fmin-0\_fmax-0-25\_order-None & sgc\_parametric\_mean\_fs-1\_fmin-0-25\_fmax-0-5\_order-None \\
sgc\_parametric\_mean\_fs-1\_fmin-1e-05\_fmax-0-5\_order-1 & sgc\_parametric\_mean\_fs-1\_fmin-0-25\_fmax-0-5\_order-1 & sgc\_parametric\_mean\_fs-1\_fmin-1e-05\_fmax-0-5\_order-20 & sgc\_parametric\_mean\_fs-1\_fmin-0\_fmax-0-25\_order-20 \\
sgc\_parametric\_mean\_fs-1\_fmin-0-25\_fmax-0-5\_order-20 & sgc\_parametric\_max\_fs-1\_fmin-1e-05\_fmax-0-5\_order-None & sgc\_parametric\_max\_fs-1\_fmin-0\_fmax-0-25\_order-None & sgc\_parametric\_max\_fs-1\_fmin-0-25\_fmax-0-5\_order-None \\
sgc\_parametric\_max\_fs-1\_fmin-0\_fmax-0-25\_order-1 & sgc\_parametric\_max\_fs-1\_fmin-0-25\_fmax-0-5\_order-1 & sgc\_parametric\_max\_fs-1\_fmin-1e-05\_fmax-0-5\_order-20 & sgc\_parametric\_max\_fs-1\_fmin-0\_fmax-0-25\_order-20 \\
sgc\_parametric\_max\_fs-1\_fmin-0-25\_fmax-0-5\_order-20 & psi\_wavelet\_mean\_fs-1\_fmin-0\_fmax-0-5\_mean & lmfit\_Ridge & lmfit\_Lasso \\
lmfit\_SGDRegressor & lmfit\_ElasticNet & lmfit\_BayesianRidge & gpfit\_DotProduct \\
gpfit\_RBF & coint\_johansen\_max\_eig\_stat\_order-0\_ardiff-10 & coint\_johansen\_trace\_stat\_order-0\_ardiff-10 & coint\_johansen\_max\_eig\_stat\_order-0\_ardiff-1 \\
coint\_johansen\_trace\_stat\_order-0\_ardiff-1 & coint\_johansen\_max\_eig\_stat\_order-1\_ardiff-10 & coint\_johansen\_trace\_stat\_order-1\_ardiff-10 & coint\_johansen\_max\_eig\_stat\_order-1\_ardiff-1 \\
coint\_johansen\_trace\_stat\_order-1\_ardiff-1 & coint\_aeg\_tstat\_trend-c\_autolag-aic\_maxlag-10 & coint\_aeg\_tstat\_trend-ct\_autolag-aic\_maxlag-10 & coint\_aeg\_tstat\_trend-ct\_autolag-bic\_maxlag-10 \\ \bottomrule
\end{tabular}
}
\end{table}

We compute all Statistical Pairwise Interactions (SPIs) using the open-source Python library \texttt{pyspi}~\citep{cliff2023unifying}, which provides a unified implementation of 284 diverse measures. To ensure numerical stability, we exclude 18 SPIs that produced invalid matrix entries (\texttt{NaN} values), resulting in a final set of $|\mathcal{S}|=266$ methods for our benchmark.

For a single multivariate time series (MTS) sample $\mathbf{X} \in \mathbb{R}^{N \times T}$, where $N$ is the number of regions of interest (ROIs) and $T$ is the number of time points, the time complexity of a typical SPI computation scales as $\mathcal{O}(N^2L)$, where $L$ reflects the method-specific internal dependency length. When considering the entire benchmarking task over a dataset $\mathcal{X}$ and the full suite of SPIs $\mathcal{S}$, the overall computational complexity lower bound becomes $\mathcal{O}(|\mathcal{X}| \cdot |\mathcal{S}| \cdot N^2L)$.

To quantify this theoretical burden in practical terms, we benchmarked each of the 266 SPIs on a representative sample (size $33 \times 240$). Using a 128 vCPU cluster as a concrete example, the resulting time distribution is shown in Figure~\ref{time_comp}. Based on these timings, we can estimate the total cost for our full dataset of $|\mathcal{X}|=4520$ samples. The time to process one sample with all 266 SPIs is approximately $18,950$ CPU-seconds (summing estimates from different time bins: $14 \text{ methods} \times 1000s + 45 \times 100s + 27 \times 10s + 180 \times 1s$). Extrapolating to the full dataset, the total computational cost is a staggering $4520 \times 18,950 \approx 8.57 \times 10^7$ CPU-seconds, equivalent to over \textbf{990 CPU-days} ($\approx 7.7$ CPU-days on a 128 vCPU). Even with access to massively parallelized cluster environments, this enormous consumption of resources makes a full benchmark practically infeasible and time-prohibitive.

This severe computational bottleneck motivates our core research question: how can we drastically reduce the number of samples while preserving a robust and reliable evaluation of the SPIs? This challenge naturally leads to our investigation of core-set selection for fMRI-based SPI benchmarking. As calculated, on a $10\%$ core-set the load reduces to $\approx99$ CPU-days ($10\%$ of the full cost). On the same $128$-vCPU cluster, this would take: $99$ CPU-days / $128$ cores = $\approx0.77$ days (or $\approx18.5$ hours). To validate our approach, we necessarily performed this exhaustive computation to establish a ground-truth ranking. However, for the purpose of evaluating core-set quality in our experiments, we restrict our analysis to a tractable subset of SPIs that take less than one second per sample, enabling rapid yet informative evaluation. The full list of 284 SPIs is summarized in Table~\ref{spis}, with those used for our core-set evaluation highlighted in bold.

\subsection{Computational Cost of Selection Methods}
\label{app:comp_cost}

To complete our analysis of efficiency, we provide an empirical comparison of the computational cost for \textbf{SCLCS} and the baseline methods. Table~\ref{tab:comp_cost_methods} details the practical time consumption required for each method. The `Time per Epoch' reflects the average wall-clock time to complete a single training epoch. The `Score Calculation Time' is the specific, one-time overhead for computing the final selection metric after the training phase is complete.

\begin{table}[ht]
\centering
\caption{Computational cost for core-set selection methods.}
\label{tab:comp_cost_methods}
\begin{tabular}{lcc}
\toprule
\textbf{Method} & \textbf{Time per Epoch (s)} & \textbf{Score Calculation Time (s)} \\
\midrule
Forgetting      & 1.8708                      & 0.0000                              \\
Entropy         & 1.9851                      & 2.1588                              \\
EL2N            & 2.5309                      & 0.3420                              \\
AUM             & 2.8389                      & 0.0957                              \\
CCS             & 2.8389                      & 11.9414                             \\
EVA             & 2.5964                      & 0.6479                              \\
BOSS            & 2.0952                      & 61.5406                             \\
\textbf{SCLCS}\textsubscript{Dense}           & 6.9296                      & 249.8204                            \\
\bottomrule
\end{tabular}
\end{table}

As the results indicate, all core-set selection methods incur a modest, one-time computational cost. This up-front investment is negligible when contrasted with the over 990 CPU-days required for a full downstream benchmark (as detailed in \textbf{Appendix~\ref{app:comps}}). It is worth noting that the score calculation for \textbf{SCLCS}\textsubscript{Dense}, which computes the \textbf{SPS} metric by measuring differences between attention matrices across epochs, represents a fixed, post-training overhead. While this step appears slower than other methods' scoring, it is a one-time process that is highly parallelizable. These findings confirm that investing a small computational budget in core-set selection is a practical and efficient solution. This validates our proposed paradigm for tackling the intractable problem of large-scale model benchmarking.

\section{Influence of Labels on Baselines}\label{app:labelef}
In the main paper, we report results of baseline methods trained using labels aligned with the evaluation objective. However, our proposed \textbf{SCLCS} framework uniformly adopts subject identity as the supervisory signal, which raises the question of training-evaluation label misalignment. To provide a more comprehensive analysis, this section further investigates the influence of such misalignment on baselines.
\subsection{Using MDD Label for Brain Fingerprinting Ranking}

\begin{table}[ht]
\centering
\caption{Comparison of brain fingerprinting ranking with subject and with MDD labels (mean $\pm$ std). Arrows indicate change under MDD supervision.}\label{Tab:mdd2bf}
\resizebox{\textwidth}{!}{
\begin{tabular}{l|l|ccc|ccc|ccc}
\toprule
\multirow{2}{*}{Method} & \multirow{2}{*}{Label} & \multicolumn{3}{c|}{nDCG@5} & \multicolumn{3}{c|}{nDCG@10} & \multicolumn{3}{c}{nDCG@20} \\
& & 0.1 & 0.3 & 0.5 & 0.1 & 0.3 & 0.5 & 0.1 & 0.3 & 0.5 \\
\midrule
\multirow{2}{*}{Forgetting} & Subject & 14.41$_{\pm\text{7.84}}$ & 54.43$_{\pm\text{13.26}}$ & 43.36$_{\pm\text{3.35}}$ & 15.49$_{\pm\text{7.83}}$ & 48.57$_{\pm\text{7.85}}$ & 49.87$_{\pm\text{4.66}}$ & 20.23$_{\pm\text{7.21}}$ & 55.80$_{\pm\text{6.72}}$ & 49.70$_{\pm\text{4.60}}$ \\
 & MDD & 14.41$_{\pm\text{7.84}}$\textcolor{green}{$\downarrow$} & 54.43$_{\pm\text{13.26}}$\textcolor{green}{$\downarrow$} & 43.36$_{\pm\text{3.34}}$\textcolor{red}{$\uparrow$} & 15.49$_{\pm\text{7.83}}$\textcolor{green}{$\downarrow$} & 48.57$_{\pm\text{7.85}}$\textcolor{green}{$\downarrow$} & 49.87$_{\pm\text{4.66}}$\textcolor{green}{$\downarrow$} & 20.23$_{\pm\text{7.21}}$\textcolor{green}{$\downarrow$} & 55.80$_{\pm\text{6.72}}$\textcolor{green}{$\downarrow$} & 49.70$_{\pm\text{4.60}}$\textcolor{green}{$\downarrow$} \\
\multirow{2}{*}{Entropy} & Subject & 47.40$_{\pm\text{40.26}}$ & 22.84$_{\pm\text{13.20}}$ & 59.95$_{\pm\text{26.99}}$ & 37.73$_{\pm\text{25.94}}$ & 32.05$_{\pm\text{15.22}}$ & 58.68$_{\pm\text{23.71}}$ & 36.05$_{\pm\text{18.11}}$ & 35.90$_{\pm\text{16.19}}$ & 57.72$_{\pm\text{21.84}}$ \\
 & MDD & 24.38$_{\pm\text{23.28}}$\textcolor{green}{$\downarrow$} & 36.60$_{\pm\text{9.47}}$\textcolor{red}{$\uparrow$} & 68.74$_{\pm\text{19.96}}$\textcolor{red}{$\uparrow$} & 29.45$_{\pm\text{28.94}}$\textcolor{green}{$\downarrow$} & 44.86$_{\pm\text{25.46}}$\textcolor{red}{$\uparrow$} & 58.27$_{\pm\text{13.07}}$\textcolor{green}{$\downarrow$} & 33.17$_{\pm\text{25.94}}$\textcolor{green}{$\downarrow$} & 46.55$_{\pm\text{13.74}}$\textcolor{red}{$\uparrow$} & 59.05$_{\pm\text{11.16}}$\textcolor{red}{$\uparrow$} \\
\multirow{2}{*}{El2N} & Subject & 35.56$_{\pm\text{36.16}}$ & 20.30$_{\pm\text{5.77}}$ & 31.03$_{\pm\text{34.57}}$ & 36.51$_{\pm\text{25.38}}$ & 33.18$_{\pm\text{14.68}}$ & 32.70$_{\pm\text{30.84}}$ & 33.30$_{\pm\text{21.55}}$ & 35.82$_{\pm\text{12.10}}$ & 40.06$_{\pm\text{32.08}}$ \\
 & MDD & 4.51$_{\pm\text{4.67}}$\textcolor{green}{$\downarrow$} & 19.27$_{\pm\text{15.05}}$\textcolor{green}{$\downarrow$} & 44.09$_{\pm\text{25.52}}$\textcolor{red}{$\uparrow$} & 14.67$_{\pm\text{20.47}}$\textcolor{green}{$\downarrow$} & 22.51$_{\pm\text{11.13}}$\textcolor{green}{$\downarrow$} & 47.17$_{\pm\text{20.08}}$\textcolor{red}{$\uparrow$} & 14.60$_{\pm\text{18.20}}$\textcolor{green}{$\downarrow$} & 24.27$_{\pm\text{10.47}}$\textcolor{green}{$\downarrow$} & 48.58$_{\pm\text{17.72}}$\textcolor{red}{$\uparrow$} \\
\multirow{2}{*}{AUM} & Subject & 65.92$_{\pm\text{33.80}}$ & 56.68$_{\pm\text{11.11}}$ & 38.17$_{\pm\text{13.94}}$ & 60.95$_{\pm\text{30.91}}$ & 62.05$_{\pm\text{4.91}}$ & 36.83$_{\pm\text{8.78}}$ & 51.75$_{\pm\text{22.42}}$ & 59.09$_{\pm\text{5.72}}$ & 38.34$_{\pm\text{6.88}}$ \\
 & MDD & 42.32$_{\pm\text{42.32}}$\textcolor{green}{$\downarrow$} & 72.58$_{\pm\text{23.18}}$\textcolor{red}{$\uparrow$} & 48.75$_{\pm\text{30.04}}$\textcolor{red}{$\uparrow$} & 36.06$_{\pm\text{29.57}}$\textcolor{green}{$\downarrow$} & 63.74$_{\pm\text{19.11}}$\textcolor{red}{$\uparrow$} & 53.60$_{\pm\text{25.75}}$\textcolor{red}{$\uparrow$} & 36.70$_{\pm\text{19.54}}$\textcolor{green}{$\downarrow$} & 65.14$_{\pm\text{17.92}}$\textcolor{red}{$\uparrow$} & 53.85$_{\pm\text{22.38}}$\textcolor{red}{$\uparrow$} \\
\multirow{2}{*}{CCS} & Subject & 1.90$_{\pm\text{2.07}}$ & 30.53$_{\pm\text{14.15}}$ & 46.65$_{\pm\text{22.32}}$ & 2.92$_{\pm\text{3.24}}$ & 29.13$_{\pm\text{8.00}}$ & 51.78$_{\pm\text{24.12}}$ & 16.24$_{\pm\text{13.34}}$ & 32.56$_{\pm\text{14.15}}$ & 52.18$_{\pm\text{20.20}}$ \\
 & MDD & 11.75$_{\pm\text{11.52}}$\textcolor{red}{$\uparrow$} & 41.01$_{\pm\text{24.51}}$\textcolor{red}{$\uparrow$} & 24.68$_{\pm\text{20.39}}$\textcolor{green}{$\downarrow$} & 12.93$_{\pm\text{12.11}}$\textcolor{red}{$\uparrow$} & 44.33$_{\pm\text{21.58}}$\textcolor{red}{$\uparrow$} & 33.11$_{\pm\text{26.79}}$\textcolor{green}{$\downarrow$} & 18.45$_{\pm\text{10.26}}$\textcolor{red}{$\uparrow$} & 44.20$_{\pm\text{22.46}}$\textcolor{red}{$\uparrow$} & 33.65$_{\pm\text{26.74}}$\textcolor{green}{$\downarrow$} \\
\multirow{2}{*}{EVA} & Subject & 38.40$_{\pm\text{40.57}}$ & 62.03$_{\pm\text{26.64}}$ & 37.80$_{\pm\text{42.45}}$ & 43.37$_{\pm\text{19.92}}$ & 55.01$_{\pm\text{20.05}}$ & 49.56$_{\pm\text{33.28}}$ & 43.22$_{\pm\text{15.28}}$ & 53.51$_{\pm\text{14.70}}$ & 65.49$_{\pm\text{21.99}}$ \\
 & MDD & 31.99$_{\pm\text{32.72}}$\textcolor{green}{$\downarrow$} & 30.57$_{\pm\text{29.76}}$\textcolor{green}{$\downarrow$} & 42.34$_{\pm\text{42.47}}$\textcolor{red}{$\uparrow$} & 30.85$_{\pm\text{30.00}}$\textcolor{green}{$\downarrow$} & 31.90$_{\pm\text{30.09}}$\textcolor{green}{$\downarrow$} & 42.55$_{\pm\text{25.48}}$\textcolor{green}{$\downarrow$} & 41.33$_{\pm\text{15.02}}$\textcolor{green}{$\downarrow$} & 28.07$_{\pm\text{23.41}}$\textcolor{green}{$\downarrow$} & 53.45$_{\pm\text{16.83}}$\textcolor{green}{$\downarrow$} \\
\multirow{2}{*}{BOSS} & Subject & 15.98$_{\pm\text{25.58}}$ & 42.11$_{\pm\text{24.33}}$ & 35.36$_{\pm\text{9.21}}$ & 29.44$_{\pm\text{11.05}}$ & 40.57$_{\pm\text{23.37}}$ & 39.15$_{\pm\text{6.71}}$ & 31.45$_{\pm\text{9.45}}$ & 38.24$_{\pm\text{19.65}}$ & 38.92$_{\pm\text{5.97}}$ \\
 & MDD & 22.90$_{\pm\text{23.99}}$\textcolor{red}{$\uparrow$} & 51.70$_{\pm\text{13.14}}$\textcolor{red}{$\uparrow$} & 40.08$_{\pm\text{31.21}}$\textcolor{red}{$\uparrow$} & 22.92$_{\pm\text{18.55}}$\textcolor{green}{$\downarrow$} & 51.87$_{\pm\text{12.74}}$\textcolor{red}{$\uparrow$} & 39.62$_{\pm\text{29.12}}$\textcolor{red}{$\uparrow$} & 29.20$_{\pm\text{17.36}}$\textcolor{green}{$\downarrow$} & 51.71$_{\pm\text{10.03}}$\textcolor{red}{$\uparrow$} & 40.14$_{\pm\text{28.44}}$\textcolor{red}{$\uparrow$} \\
\bottomrule
\end{tabular}}
\end{table}

We first evaluate the performance of baseline methods trained with MDD diagnosis labels for core-set selection in the brain fingerprinting ranking task, using the results from subject-identity supervision (as reported in the main paper) as reference. This alternative labeling scheme aligns with the intended design of most baseline algorithms, which aim to select core-sets based on the target task labels (i.e., MDD vs. HC diagnosis labels from the REST-meta-MDD dataset).

As shown in Table~\ref{Tab:mdd2bf}, this supervision shift generally leads to performance degradation across most methods, suggesting a potential mismatch between the binary nature of MDD labels and the requirements of brain fingerprinting, which involves a one-vs-all subject-level identification task. Specifically, training with MDD labels provides weaker sample-level supervision due to reduced class granularity, potentially limiting the diversity captured during core-set selection. Consequently, the selected samples may fail to adequately support the subject-wise discriminative capacity of SPIs.

Nevertheless, certain methods such as AUM and BOSS exhibit improved ranking stability at higher sampling ratios. This may be attributed to their scoring strategies, which explicitly account for sample diversity or informativeness. Under this paradigm, incorporating diagnosis-based labels introduces an additional semantic dimension that enhances the selection process, enabling these methods to better preserve the representational structure of the full dataset.

\subsection{Using Subject Label for MDD Ranking}

\begin{table}[ht]
\centering
\caption{Comparison of MDD diagnosis ranking with MDD labels and with subject identities(mean $\pm$ std). Arrows indicate change under subject supervision.}
\resizebox{\textwidth}{!}{
\begin{tabular}{l|l|ccc|ccc|ccc}
\toprule
\multirow{2}{*}{Method} & \multirow{2}{*}{Label} & \multicolumn{3}{c|}{nDCG@5} & \multicolumn{3}{c|}{nDCG@10} & \multicolumn{3}{c}{nDCG@20} \\
& & 0.1 & 0.3 & 0.5 & 0.1 & 0.3 & 0.5 & 0.1 & 0.3 & 0.5 \\
\midrule
\multirow{2}{*}{Forgetting} & MDD & 28.70$_{\pm\text{28.06}}$ & 41.60$_{\pm\text{28.08}}$ & 61.42$_{\pm\text{4.31}}$ & 40.56$_{\pm\text{28.06}}$ & 50.27$_{\pm\text{32.37}}$ & 66.72$_{\pm\text{8.60}}$ & 44.37$_{\pm\text{28.07}}$ & 57.02$_{\pm\text{33.39}}$ & 75.70$_{\pm\text{7.58}}$ \\
 & Subject & 24.64$_{\pm\text{28.64}}$\textcolor{green}{$\downarrow$} & 43.69$_{\pm\text{36.57}}$\textcolor{red}{$\uparrow$} & 87.74$_{\pm\text{10.67}}$\textcolor{red}{$\uparrow$} & 28.52$_{\pm\text{30.48}}$\textcolor{green}{$\downarrow$} & 46.33$_{\pm\text{37.31}}$\textcolor{green}{$\downarrow$} & 88.57$_{\pm\text{11.93}}$\textcolor{red}{$\uparrow$} & 32.88$_{\pm\text{32.35}}$\textcolor{green}{$\downarrow$} & 51.60$_{\pm\text{38.74}}$\textcolor{green}{$\downarrow$} & 90.93$_{\pm\text{8.89}}$\textcolor{red}{$\uparrow$} \\
\multirow{2}{*}{Entropy} & MDD & 29.00$_{\pm\text{45.10}}$ & 31.84$_{\pm\text{46.19}}$ & 57.48$_{\pm\text{29.81}}$ & 29.17$_{\pm\text{43.24}}$ & 40.75$_{\pm\text{44.05}}$ & 63.23$_{\pm\text{23.44}}$ & 30.30$_{\pm\text{43.52}}$ & 45.85$_{\pm\text{45.00}}$ & 70.79$_{\pm\text{18.00}}$ \\
 & Subject & 53.09$_{\pm\text{17.82}}$\textcolor{red}{$\uparrow$} & 46.99$_{\pm\text{34.99}}$\textcolor{red}{$\uparrow$} & 66.46$_{\pm\text{19.72}}$\textcolor{red}{$\uparrow$} & 58.79$_{\pm\text{16.78}}$\textcolor{red}{$\uparrow$} & 52.25$_{\pm\text{36.91}}$\textcolor{red}{$\uparrow$} & 75.19$_{\pm\text{18.27}}$\textcolor{red}{$\uparrow$} & 65.36$_{\pm\text{13.54}}$\textcolor{red}{$\uparrow$} & 56.61$_{\pm\text{37.02}}$\textcolor{red}{$\uparrow$} & 81.40$_{\pm\text{14.68}}$\textcolor{red}{$\uparrow$} \\
\multirow{2}{*}{El2N} & MDD & 30.93$_{\pm\text{41.17}}$ & 51.70$_{\pm\text{28.14}}$ & 68.85$_{\pm\text{12.46}}$ & 36.05$_{\pm\text{24.41}}$ & 58.54$_{\pm\text{27.55}}$ & 72.96$_{\pm\text{15.77}}$ & 41.04$_{\pm\text{39.61}}$ & 64.68$_{\pm\text{26.25}}$ & 78.85$_{\pm\text{14.56}}$ \\
 & Subject & 68.75$_{\pm\text{21.44}}$\textcolor{red}{$\uparrow$} & 42.90$_{\pm\text{32.24}}$\textcolor{green}{$\downarrow$} & 60.18$_{\pm\text{17.44}}$\textcolor{green}{$\downarrow$} & 71.38$_{\pm\text{22.07}}$\textcolor{red}{$\uparrow$} & 50.92$_{\pm\text{30.16}}$\textcolor{green}{$\downarrow$} & 66.67$_{\pm\text{18.04}}$\textcolor{green}{$\downarrow$} & 77.12$_{\pm\text{16.54}}$\textcolor{red}{$\uparrow$} & 56.72$_{\pm\text{28.89}}$\textcolor{green}{$\downarrow$} & 73.55$_{\pm\text{15.61}}$\textcolor{green}{$\downarrow$} \\
\multirow{2}{*}{AUM} & MDD & 36.07$_{\pm\text{12.59}}$ & 42.68$_{\pm\text{46.16}}$ & 35.49$_{\pm\text{14.37}}$ & 39.17$_{\pm\text{14.58}}$ & 44.41$_{\pm\text{44.05}}$ & 58.14$_{\pm\text{16.54}}$ & 44.94$_{\pm\text{18.48}}$ & 48.44$_{\pm\text{42.48}}$ & 64.36$_{\pm\text{14.40}}$ \\
 & Subject & 41.59$_{\pm\text{10.89}}$\textcolor{red}{$\uparrow$} & 52.72$_{\pm\text{39.35}}$\textcolor{red}{$\uparrow$} & 65.20$_{\pm\text{46.67}}$\textcolor{red}{$\uparrow$} & 53.62$_{\pm\text{5.24}}$\textcolor{red}{$\uparrow$} & 58.66$_{\pm\text{37.39}}$\textcolor{red}{$\uparrow$} & 67.17$_{\pm\text{44.70}}$\textcolor{red}{$\uparrow$} & 53.62$_{\pm\text{5.24}}$\textcolor{red}{$\uparrow$} & 65.26$_{\pm\text{36.40}}$\textcolor{red}{$\uparrow$} & 70.07$_{\pm\text{41.78}}$\textcolor{red}{$\uparrow$} \\
\multirow{2}{*}{CCS} & MDD & 55.95$_{\pm\text{16.28}}$ & 59.92$_{\pm\text{30.08}}$ & 74.73$_{\pm\text{12.46}}$ & 59.25$_{\pm\text{16.26}}$ & 67.01$_{\pm\text{24.61}}$ & 75.41$_{\pm\text{18.08}}$ & 68.98$_{\pm\text{15.67}}$ & 71.62$_{\pm\text{21.15}}$ & 78.77$_{\pm\text{13.49}}$ \\
 & Subject & 47.81$_{\pm\text{27.72}}$\textcolor{green}{$\downarrow$} & 25.25$_{\pm\text{8.66}}$\textcolor{green}{$\downarrow$} & 56.95$_{\pm\text{27.90}}$\textcolor{green}{$\downarrow$} & 55.46$_{\pm\text{21.63}}$\textcolor{green}{$\downarrow$} & 32.40$_{\pm\text{8.73}}$\textcolor{green}{$\downarrow$} & 62.59$_{\pm\text{25.19}}$\textcolor{green}{$\downarrow$} & 62.04$_{\pm\text{17.97}}$\textcolor{green}{$\downarrow$} & 41.36$_{\pm\text{7.82}}$\textcolor{green}{$\downarrow$} & 69.14$_{\pm\text{19.28}}$\textcolor{green}{$\downarrow$} \\
\multirow{2}{*}{EVA} & MDD & 31.80$_{\pm\text{17.70}}$ & 66.81$_{\pm\text{9.12}}$ & 70.07$_{\pm\text{6.05}}$ & 36.11$_{\pm\text{14.79}}$ & 72.61$_{\pm\text{8.73}}$ & 76.26$_{\pm\text{6.17}}$ & 48.39$_{\pm\text{19.48}}$ & 75.34$_{\pm\text{8.48}}$ & 81.73$_{\pm\text{5.61}}$ \\
 & Subject & 21.25$_{\pm\text{11.48}}$\textcolor{green}{$\downarrow$} & 67.32$_{\pm\text{30.66}}$\textcolor{red}{$\uparrow$} & 75.63$_{\pm\text{16.84}}$\textcolor{red}{$\uparrow$} & 28.40$_{\pm\text{10.75}}$\textcolor{green}{$\downarrow$} & 73.84$_{\pm\text{21.80}}$\textcolor{red}{$\uparrow$} & 79.06$_{\pm\text{15.46}}$\textcolor{red}{$\uparrow$} & 36.98$_{\pm\text{10.08}}$\textcolor{green}{$\downarrow$} & 80.82$_{\pm\text{15.48}}$\textcolor{red}{$\uparrow$} & 83.18$_{\pm\text{12.42}}$\textcolor{red}{$\uparrow$} \\
\multirow{2}{*}{BOSS} & MDD & 42.44$_{\pm\text{24.37}}$ & 57.52$_{\pm\text{20.11}}$ & 79.57$_{\pm\text{16.62}}$ & 50.64$_{\pm\text{25.12}}$ & 64.86$_{\pm\text{21.86}}$ & 84.70$_{\pm\text{14.07}}$ & 58.11$_{\pm\text{23.22}}$ & 71.36$_{\pm\text{18.94}}$ & 88.95$_{\pm\text{9.49}}$ \\
 & Subject & 39.60$_{\pm\text{37.05}}$\textcolor{green}{$\downarrow$} & 67.40$_{\pm\text{13.42}}$\textcolor{red}{$\uparrow$} & 79.51$_{\pm\text{17.82}}$\textcolor{green}{$\downarrow$} & 45.30$_{\pm\text{31.25}}$\textcolor{green}{$\downarrow$} & 68.47$_{\pm\text{15.58}}$\textcolor{red}{$\uparrow$} & 86.21$_{\pm\text{9.45}}$\textcolor{red}{$\uparrow$} & 54.29$_{\pm\text{2.51}}$\textcolor{green}{$\downarrow$} & 77.59$_{\pm\text{11.02}}$\textcolor{red}{$\uparrow$} & 87.29$_{\pm\text{8.34}}$\textcolor{green}{$\downarrow$} \\
\bottomrule
\end{tabular}}
\end{table}

As a complementary analysis, we also evaluate the performance of baselines trained with subject identity labels for core-set selection in the MDD diagnosis ranking task. This setting closely aligns with our proposed SCLCS framework and allows us to investigate whether individual-level supervision provides a stronger basis for core-set construction.

Compared to the MDD-supervised setting, we observe that a greater number of methods benefit from improved ranking stability under subject identity supervision. For instance, Entropy, AUM, and El2N exhibit consistent performance gains at higher sampling ratios (e.g., AUM improves from 58.14\% to 67.17\% in nDCG@10@0.5, and Entropy from 63.23\% to 75.19\%). This trend suggests that supervision aligned with subject-level heterogeneity may better preserve fine-grained information necessary for identifying high-quality representative samples.

Despite these improvements, a clear performance gap remains between all baselines and our method, as shown in Table~\ref{tab:MRT}. This highlights the non-trivial advantage of \textbf{SCLCS}, where structural perturbation scoring and density-aware sampling jointly enforce both informativeness and diversity in selected subsets.

Interestingly, AUM shows consistent gains across nearly all metrics, suggesting that its original scoring, formulated under task-specific supervision, may underestimate structural variation among samples. Subject-based training appears to compensate for this limitation by injecting more diverse contrastive signals during scoring, revealing a potential direction for enhancing its robustness.

Entropy also achieves performance gains across the board, with improvements as high as 11.6\% in nDCG@20 (from 70.79\% to 81.40\%). However, as shown in Figure~\ref{balance}, Entropy consistently selects highly imbalanced subsets regardless of supervision label, often dominated by a small number of subjects. This structural bias undermines its utility for benchmarking core-set methods, as it fails to preserve a representative distribution of the dataset. Thus, despite the numerical improvements, Entropy remains unsuitable for core-set-based SPI evaluation.

These findings reinforce the need to consider both label alignment and structural diversity in core-set selection. Subject-level supervision offers a promising direction, but our method’s explicit modeling of structure-aware consistency and coverage remains critical for reliable benchmarking.

\section{Influence of the Proposed Density-balanced Sampling on Baselines}\label{app:density-bal}
\begin{figure}[ht!]
  \centering
  \includegraphics[width=0.9\linewidth]{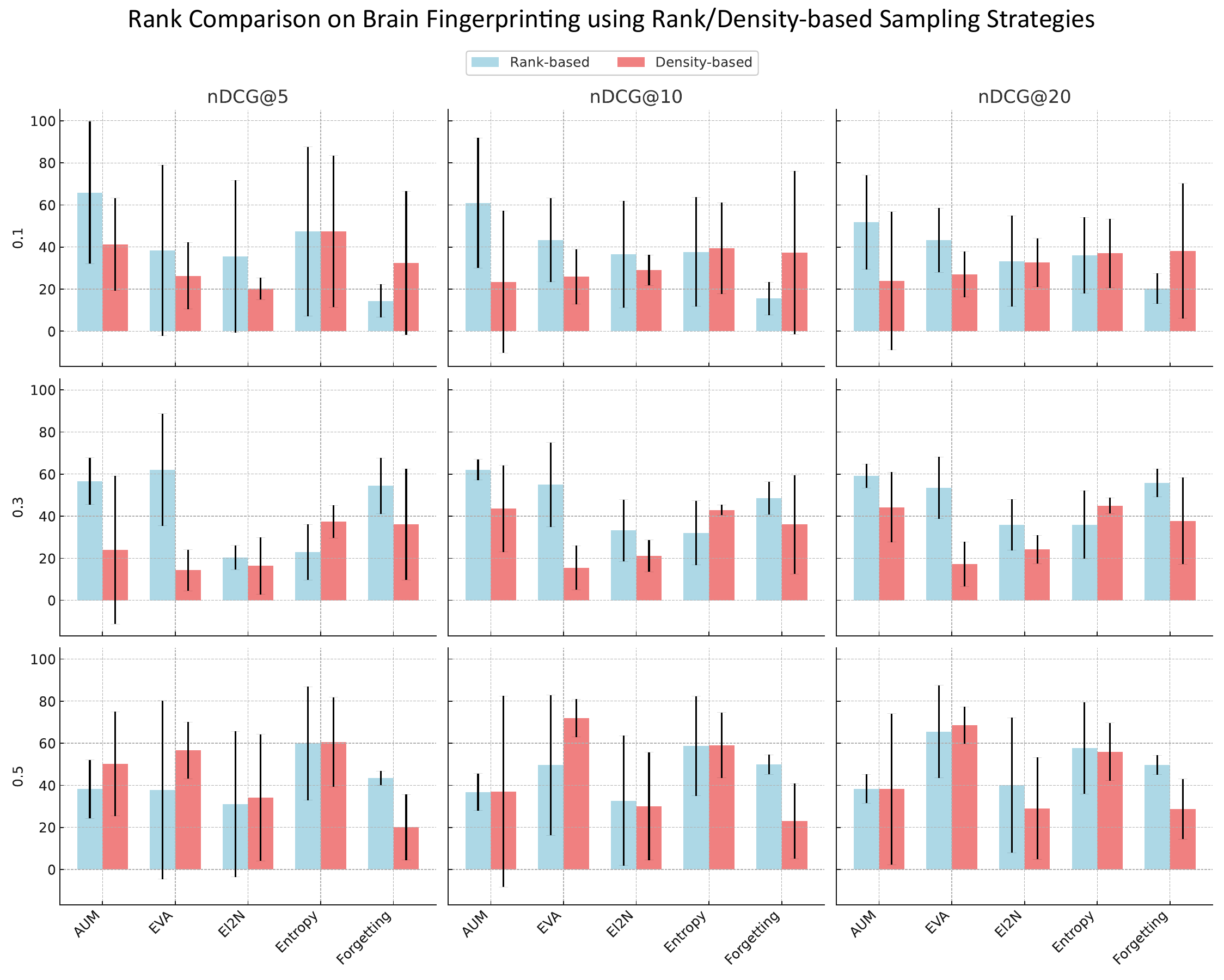}
  \caption{Rank comparison on brain fingerprinting using rank/density-based sampling strategies.}
  \label{fig:rankdensebfcomp}
\end{figure}

\begin{figure}[ht!]
  \centering
  \includegraphics[width=0.9\linewidth]{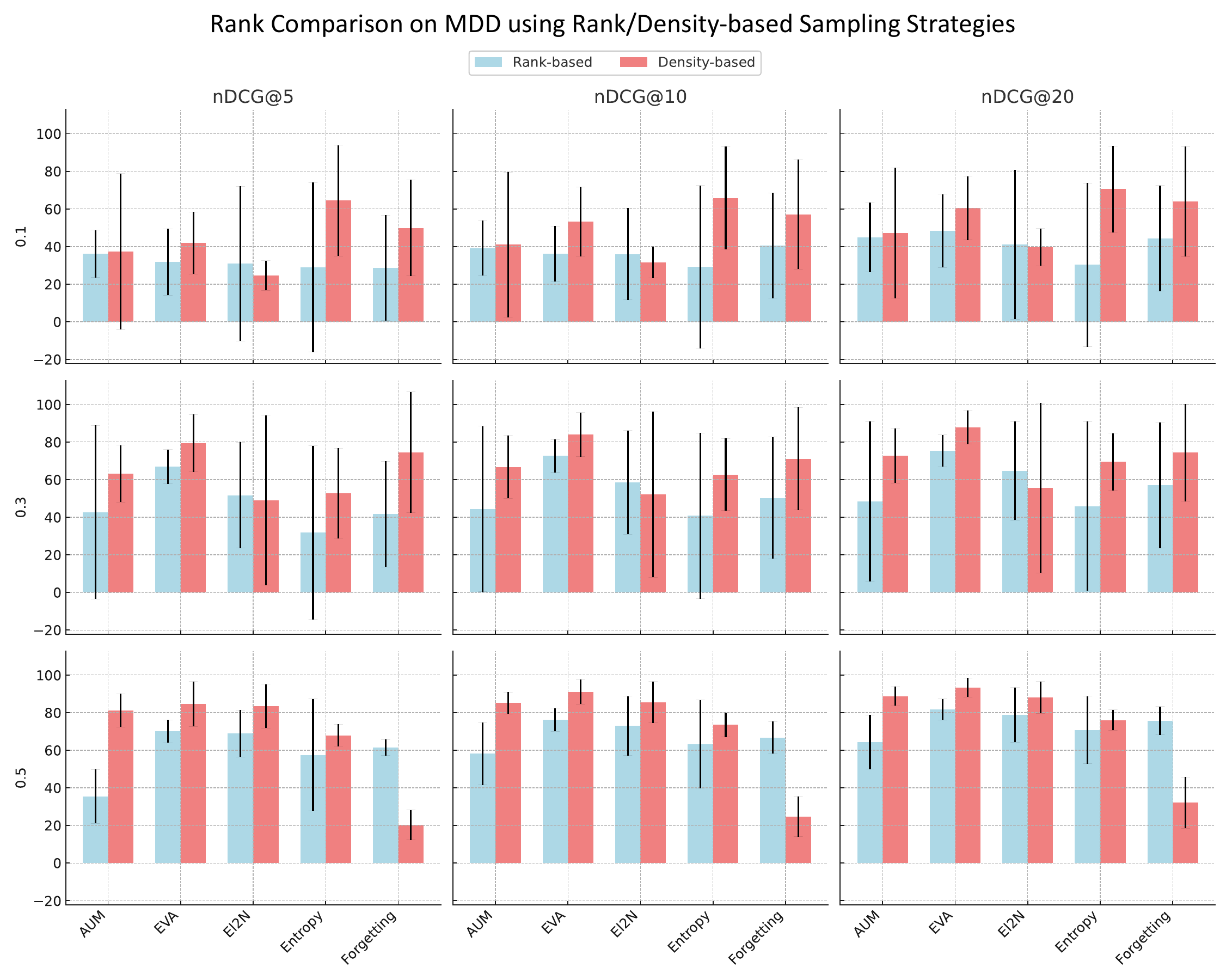}
  \caption{Rank comparison on MDD diagnosis using rank/density-based sampling strategies.}
  \label{fig:rankdensemdcomp}
\end{figure}
As stated in \textbf{Section~\ref{sec:structure_aware_sampling}}, our proposed density-based sampling strategy is not only central to the \textbf{SCLCS} framework, but also generalizable to other score-based core-set selection methods. To evaluate its applicability beyond our method, we visualize comparative results on the Brain Fingerprinting and MDD Diagnosis tasks in Figure~\ref{fig:rankdensebfcomp} and Figure~\ref{fig:rankdensemdcomp}, respectively. We exclude CCS and BOSS from this analysis due to their use of task-specific sampling designs.

On the Brain Fingerprinting task, density-based sampling frequently alters SPI ranking performance across baselines. For example, El2N underperforms in multiple metrics when density is applied (e.g., nDCG@5 at ratio 0.1 drops from $\sim$40 to $\sim$20), while EVA shows inconsistent trends across sampling ratios. This instability may be attributed to the fact that, when supervised with subject identity, score-based methods tend to prioritize samples that support subject-level diversity. Replacing this priority with a density-based criterion may inadvertently distort the structural balance of the selected subset.

Conversely, in the MDD Diagnosis setting, density-based sampling consistently improves performance. Baselines such as EVA and Entropy show marked gains in ranking stability, with EVA's nDCG@10 increasing from $\sim$70 to $\sim$90 at sampling ratio 0.5. This contrast suggests that when supervision involves fewer discrete classes (e.g., binary labels), the density structure becomes easier to estimate and more semantically aligned with the downstream evaluation objective.

An interesting deviation is observed in the Forgetting method. At high sampling ratios in both tasks, its performance noticeably declines under density-based sampling. This may stem from Forgetting’s underlying assumption: that low-confidence samples correspond to noisy or uninformative data, which does not hold well for fMRI. Due to the complex nature of brain dynamics, such samples may actually be densely clustered and structurally meaningful. Consequently, density-based sampling could overemphasize regions marked by high forgetting scores, thereby degrading ranking consistency.

In summary, density-based sampling demonstrates strong compatibility with several baseline strategies, highlighting its potential as a general-purpose augmentation. However, the interaction between density criteria and different scoring heuristics can be nontrivial, and may lead to unintended trade-offs in performance. These results underscore the importance of further empirical studies to better understand the conditions determining whether density-based selection benefits or interferes with score-driven core-set construction.

\section{Empirical Results for Theorems}\label{app:theorem_exp}
\begin{figure}[!th]
\centering
\includegraphics[width=\textwidth]{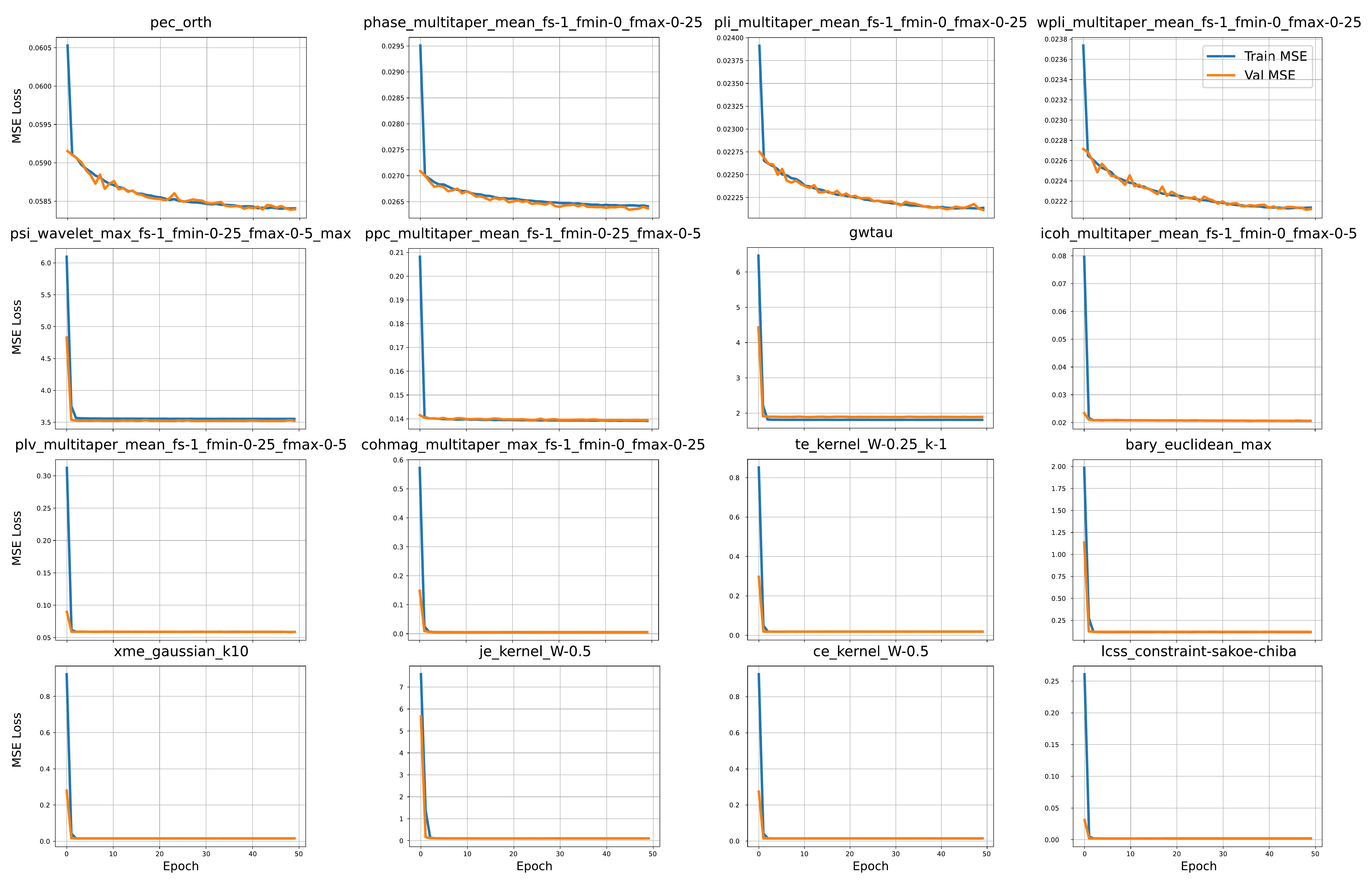}
\caption{Training and validation MSE loss convergence curves for the 16 SPI operators used in the empirical validation of \textbf{Theorem 2}.}
\label{fig:convergence_universal}
\end{figure}

\subsection{Convergence of Universal Approximation}\label{convuni}
\paragraph{Experimental Setup}
To empirically validate the universal approximation capability of our modified Transformer architecture, as posited in \textbf{Theorem~\ref{thm:universal}}, we designed and conducted a direct fitting experiment. The theorem states that our model architecture is, in principle, capable of approximating any continuous Statistical Pairwise Interaction (SPI) operator.

The core objective of this experiment was to train instances of our model to directly mimic the Functional Connectivity (FC) matrices produced by specific SPI operators. We trained a separate model instance for each operator in a diverse set of 16 representative SPIs. The experimental setup was as follows:

\begin{itemize}
\item \textbf{Input}: The input for each model was an fMRI time series sample $\mathbf{X}\in \mathbb{R}^{N\times T}$ , consistent with the primary data used in our paper.
\item \textbf{Target}: For each model, the learning target was the "ground-truth" FC matrix $\mathbf{A}_S\in \mathbb{R}^{N\times N}$, computed for that specific fMRI sample using the corresponding SPI operator.
\item \textbf{Loss Function}: We optimized the model parameters by directly minimizing the Mean Squared Error (MSE) between the model's output adjacency matrix $\mathbf{A}$ and the target matrix $\mathbf{A}_S$. Specifically, we use $\operatorname{MSE}((\mathbf{A}+\mathbf{A}^T)/2,(\mathbf{A}_S+\mathbf{A}_S^T)/2)$ to ensure the symmetry of the adjacencies.
\item \textbf{Training and Validation}: We partitioned the dataset into training, validation, and testing sets with a 70\%/10\%/20\% split. During training, we employed early stopping (patience=10) based on the validation set performance to ensure the model learned a generalizable transformation rather than merely overfitting to the training data.
\end{itemize}

\paragraph{Results and Analysis}

Figure~\ref{fig:convergence_universal} displays the convergence curves from the training process for all 16 SPI operators. Each subplot in the figure represents an independent model instance, with the x-axis denoting the training epoch and the y-axis representing the MSE Loss. The blue line indicates the MSE on the training set, while the orange line represents the MSE on the validation set.

As the plots clearly demonstrate, our model architecture exhibits a strong capacity to fit all 16 SPIs, regardless of their diverse underlying computational principles. Key observations include:

\begin{enumerate}
\item \textbf{Successful Convergence}: In all experiments, both the training and validation MSE decrease rapidly from a high initial value and eventually converge to a stable, low level. This indicates that the optimization process was successful and that the model effectively learned the mapping from the fMRI time series to the target FC matrix.
\item \textbf{Good Generalization}: The validation loss curves closely track the training loss curves without significant divergence. This confirms that the models did not overfit and that the learned approximations generalize well to unseen data.
\end{enumerate}

These convergence curves, combined with the low final test MSE values reported in Table~\ref{tab:theorem2_validation} of the main text, provide strong empirical support for \textbf{Theorem~\ref{thm:universal}}. The results collectively confirm that our proposed modified Transformer architecture possesses the practical expressive power required to represent the diverse functional forms inherent to our benchmarking task.

\subsection{Stationary and Convergence Analysis}\label{app:staconv}

\begin{figure}[!ht]
  \centering
  \includegraphics[width=0.8\linewidth]{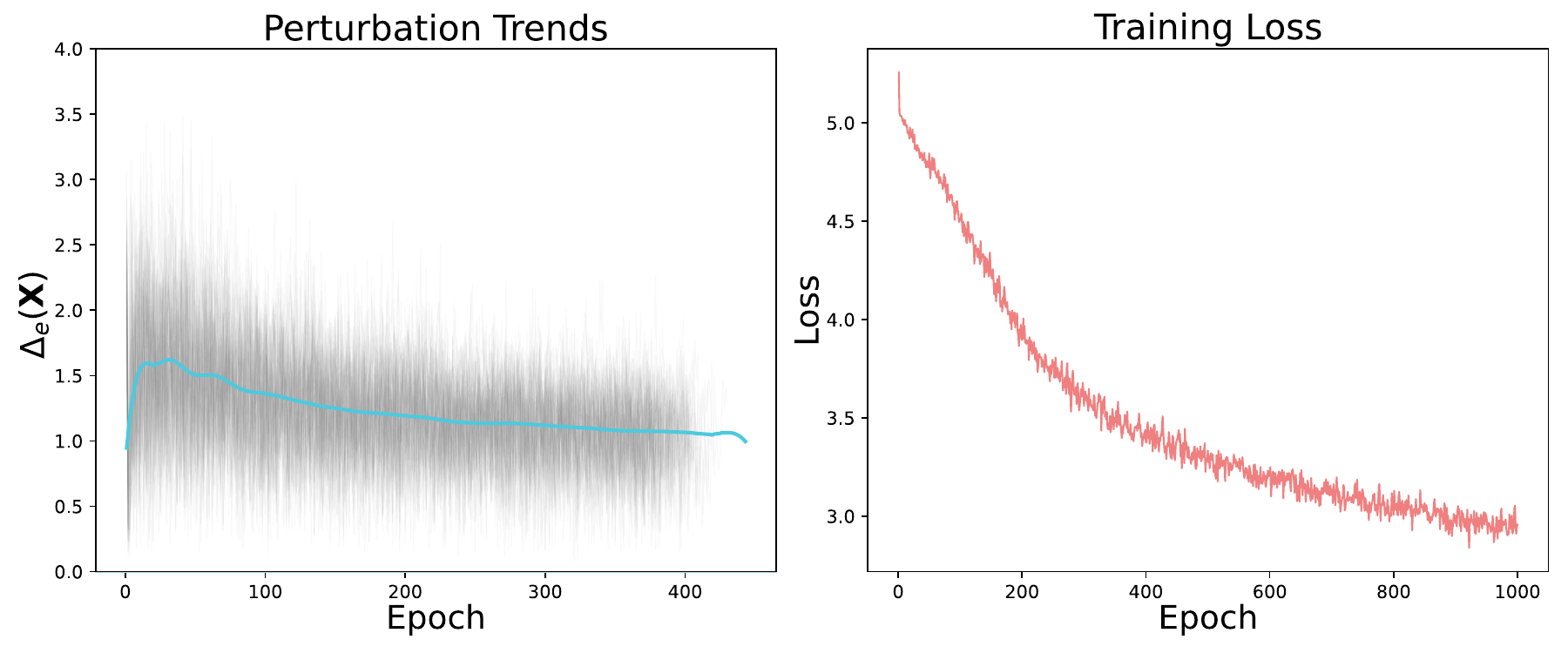}
  \caption{Convergence dynamics of \textbf{SCLCS}. \textbf{(Left)} Perturbation trends show that the mean \textbf{SPS} stabilizes relatively early in training. \textbf{(Right)} The training loss converges more gradually over $1000$ epochs. The different x-axes are used to visualize the distinct convergence timescales of each metric.}
  \label{fig:convergence}
\end{figure}

As demonstrated in \textbf{Lemma~\ref{lem:sps_consistency}}, the reliability of the Structural Perturbation Score (\textbf{SPS}) relies on the assumption that the per-epoch difference $\Delta_e(\mathbf{X}) = \|\mathbf{A}^{(e)}_{(\mathbf{X})}-\mathbf{A}^{(e-1)}_{(\mathbf{X})}\|_F^2$ is stationary and ergodic. Since our training does not incorporate validation-based early stopping, we use model convergence as an implicit criterion for termination. To empirically support this, Figure~\ref{fig:convergence} visualizes the convergence trends of structural perturbation and training loss.

Notably, these two metrics converge on different effective timescales. The left panel shows that the mean \textbf{SPS} (blue curve), which reflects the stability of the learned connectivity structures, reaches a stable plateau relatively early (around epoch 500). The initial spike in perturbation reflects an expected "burn-in" phase before the model learns stable representations. In contrast, the right panel shows that the training loss continues to decrease more gradually over the full 1000 epochs as the model makes finer adjustments to the embedding space to fully optimize the contrastive objective.

The clear convergence of both metrics, despite their different timescales, provides strong empirical validation for the assumptions in \textbf{Lemma~\ref{lem:sps_consistency}}. This confirms that \textbf{SPS} is a stable and consistent measure of structural influence, and the overall loss convergence supports the robustness of our end-to-end training pipeline for core-set selection.

\section{Generalization and Robustness Analysis}
\label{app:GandR}
This section provides experiment results to address the concerns regarding site-level generalization and SPI subset bias to prove the impact of the proposed task.
\subsection{Extension on SPI Properties}

\begin{table*}[t]
\centering
\caption{Performance comparison: brain fingerprinting vs. MDD diagnosis}
\label{tab:side_by_side_comparison}
\scriptsize 
\setlength{\tabcolsep}{2pt} 

\resizebox{\textwidth}{!}{
\begin{tabular}{l|ccccccccc|ccccccccc}
\toprule
\textbf{Method} & \multicolumn{9}{c|}{\textbf{Brain Fingerprinting}} & \multicolumn{9}{c}{\textbf{MDD Diagnosis}} \\
 & \multicolumn{3}{c}{nDCG@5} & \multicolumn{3}{c}{nDCG@10} & \multicolumn{3}{c|}{nDCG@20} & \multicolumn{3}{c}{nDCG@5} & \multicolumn{3}{c}{nDCG@10} & \multicolumn{3}{c}{nDCG@20} \\
\cmidrule(lr){2-4} \cmidrule(lr){5-7} \cmidrule(lr){8-10} \cmidrule(lr){11-13} \cmidrule(lr){14-16} \cmidrule(lr){17-19}
Ratio & 0.1 & 0.3 & 0.5 & 0.1 & 0.3 & 0.5 & 0.1 & 0.3 & 0.5 & 0.1 & 0.3 & 0.5 & 0.1 & 0.3 & 0.5 & 0.1 & 0.3 & 0.5 \\
\midrule
\multicolumn{19}{l}{\textbf{SECTION A: Slow-Only SPIs}} \\
\midrule
Random       & 0.883 & 0.863 & 0.805 & 0.903 & 0.898 & 0.835 & 0.921 & 0.923 & 0.863 & 0.946 & 0.998 & 0.997 & 0.962 & 0.998 & 0.998 & 0.975 & 0.999 & 0.998 \\
Forgetting   & 0.591 & 0.441 & 0.837 & 0.614 & 0.478 & 0.871 & 0.642 & 0.514 & 0.898 & 0.781 & 0.946 & 0.997 & 0.839 & 0.965 & 0.998 & 0.886 & 0.977 & 0.999 \\
Entropy      & 0.544 & 0.963 & 0.907 & 0.615 & 0.970 & 0.977 & 0.683 & 0.975 & 0.982 & 0.604 & 0.636 & 0.941 & 0.619 & 0.640 & 0.959 & 0.628 & 0.652 & 0.973 \\
EI2N         & 0.887 & 0.852 & 0.928 & 0.912 & 0.884 & 0.945 & 0.928 & 0.909 & 0.958 & 0.967 & 0.897 & 0.994 & 0.976 & 0.930 & 0.996 & 0.983 & 0.957 & 0.997 \\
AUM          & 0.835 & 0.983 & 0.808 & 0.866 & 0.984 & 0.848 & 0.905 & 0.983 & 0.881 & 0.934 & 0.849 & 0.933 & 0.953 & 0.890 & 0.955 & 0.963 & 0.923 & 0.971 \\
CCS          & 0.785 & 0.737 & 0.845 & 0.822 & 0.790 & 0.880 & 0.859 & 0.832 & 0.908 & 0.840 & 0.966 & 0.999 & 0.887 & 0.978 & 0.999 & 0.923 & 0.985 & 0.999 \\
EVA          & 0.628 & 0.906 & 0.650 & 0.663 & 0.923 & 0.707 & 0.701 & 0.935 & 0.761 & 0.795 & 0.997 & 0.999 & 0.848 & 0.998 & 0.999 & 0.895 & 0.999 & 0.999 \\
BOSS         & 0.385 & 0.956 & 0.602 & 0.444 & 0.951 & 0.656 & 0.521 & 0.945 & 0.710 & 0.636 & 0.968 & 0.992 & 0.760 & 0.976 & 0.995 & 0.843 & 0.988 & 0.995 \\
\textbf{SCLCS} & 0.879 & 0.941 & 0.999 & 0.907 & 0.955 & 0.999 & 0.932 & 0.961 & 0.999 & 0.874 & 0.983 & 0.997 & 0.916 & 0.989 & 0.998 & 0.945 & 0.993 & 0.999 \\
\textbf{SCLCS}$_{\text{Dense}}$ & 0.903 & 0.962 & 0.943 & 0.922 & 0.967 & 0.956 & 0.935 & 0.969 & 0.964 & 0.954 & 0.999 & 0.999 & 0.969 & 0.999 & 0.999 & 0.979 & 0.999 & 0.999 \\
\midrule
\multicolumn{19}{l}{\textbf{SECTION B: Mixture SPIs}} \\
\midrule
Random       & 0.280 & 0.680 & 0.646 & 0.389 & 0.689 & 0.715 & 0.392 & 0.765 & 0.722 & 0.504 & 0.242 & 0.718 & 0.601 & 0.263 & 0.779 & 0.633 & 0.306 & 0.828 \\
Forgetting   & 0.323 & 0.464 & 0.711 & 0.421 & 0.485 & 0.742 & 0.465 & 0.497 & 0.749 & 0.255 & 0.415 & 0.634 & 0.380 & 0.502 & 0.670 & 0.432 & 0.563 & 0.758 \\
Entropy      & 0.442 & 0.660 & 0.663 & 0.485 & 0.656 & 0.678 & 0.551 & 0.668 & 0.684 & 0.289 & 0.327 & 0.565 & 0.291 & 0.414 & 0.621 & 0.307 & 0.467 & 0.702 \\
EI2N         & 0.485 & 0.441 & 0.531 & 0.552 & 0.515 & 0.570 & 0.564 & 0.570 & 0.623 & 0.310 & 0.522 & 0.656 & 0.360 & 0.576 & 0.705 & 0.416 & 0.646 & 0.761 \\
AUM          & 0.491 & 0.680 & 0.870 & 0.565 & 0.737 & 0.882 & 0.604 & 0.767 & 0.877 & 0.361 & 0.415 & 0.536 & 0.403 & 0.435 & 0.585 & 0.452 & 0.478 & 0.640 \\
CCS          & 0.228 & 0.480 & 0.495 & 0.345 & 0.478 & 0.507 & 0.352 & 0.576 & 0.549 & 0.469 & 0.599 & 0.747 & 0.519 & 0.670 & 0.754 & 0.628 & 0.699 & 0.790 \\
EVA          & 0.193 & 0.527 & 0.601 & 0.362 & 0.597 & 0.622 & 0.396 & 0.640 & 0.655 & 0.274 & 0.529 & 0.688 & 0.353 & 0.570 & 0.755 & 0.440 & 0.616 & 0.812 \\
BOSS         & 0.438 & 0.592 & 0.505 & 0.433 & 0.662 & 0.531 & 0.428 & 0.687 & 0.561 & 0.373 & 0.620 & 0.796 & 0.435 & 0.620 & 0.859 & 0.519 & 0.721 & 0.879 \\
\textbf{SCLCS} & 0.790 & 0.704 & 0.785 & 0.795 & 0.734 & 0.807 & 0.823 & 0.755 & 0.819 & 0.279 & 0.703 & 0.631 & 0.303 & 0.708 & 0.613 & 0.368 & 0.764 & 0.680 \\
\textbf{SCLCS}$_{\text{Dense}}$ & 0.442 & 0.758 & 0.812 & 0.552 & 0.751 & 0.827 & 0.576 & 0.750 & 0.827 & 0.447 & 0.449 & 0.716 & 0.522 & 0.533 & 0.706 & 0.621 & 0.594 & 0.752 \\
\bottomrule
\end{tabular}
}
\end{table*}

We create two new benchmark sets: (1) A "Slower-Only" set containing 50 SPIs with $>1$s compute time (up to $10$s), and (2) A "Mixture" set combining these $50$ slow SPIs with our original $130$ fast SPIs (Total = $180$). Results are reported in Table~\ref{tab:side_by_side_comparison}.

The findings are two-fold:
\begin{itemize}
    \item On "Slow-Only": Surprisingly, the ranking task becomes "easier." Even the Random baseline performs robustly (e.g., nDCG@5 > 0.9 on the MDD task). This suggests that for this specific subset of slower methods, performance is less sensitive to specific data selection.
    \item On "Mixture" (Full Benchmark): However, when combining fast and slow \textbf{SPIs}, \textbf{SCLCS} re-emerges as the clear winner. On the Brain Fingerprinting task ($10\%$ ratio), \textbf{SCLCS} achieves an nDCG@5 of $0.79$, whereas Random collapses to $0.28$.
\end{itemize}

It proves that simple heuristics (like Random) are brittle and fail in realistic, heterogeneous benchmarking scenarios (the "Mixture" case). \textbf{SCLCS} is necessary because it robustly handles the full spectrum of SPI behaviors, effectively weighting the top-performing methods that researchers care about most. More importantly, it suggests that different properties of the candidate models may affect the difficulty of the proposed problem. Exploring such effect is a promising future direction.

\begin{table*}[h]
\centering
\caption{Site balance on core-set (L1 distance, lower better)}
\label{tab:site_balance}
\footnotesize
\resizebox{\textwidth}{!}{
\begin{tabular}{lcccccc}
\toprule
\textbf{Method} & \multicolumn{2}{c}{\textbf{L1 Distance @ 0.1 ratio}} & \multicolumn{2}{c}{\textbf{L1 Distance @ 0.3 ratio}} & \multicolumn{2}{c}{\textbf{L1 Distance @ 0.5 ratio}} \\
\cmidrule(lr){2-3} \cmidrule(lr){4-5} \cmidrule(lr){6-7}
 & Sample Level & Subject Level & Sample Level & Subject Level & Sample Level & Subject Level \\
\midrule
Forgetting   & 0.239 & 0.243 & 0.184 & 0.107 & 0.126 & 0.048 \\
Entropy      & 0.186 & 0.139 & 0.149 & 0.112 & 0.146 & 0.157 \\
EI2N         & 0.303 & 0.350 & 0.260 & 0.220 & 0.235 & 0.195 \\
AUM          & 0.157 & 0.143 & 0.094 & 0.181 & 0.057 & 0.176 \\
CCS          & 0.202 & 0.183 & 0.143 & 0.175 & 0.103 & 0.148 \\
EVA          & 0.111 & 0.101 & 0.047 & 0.017 & 0.024 & 0.081 \\
BOSS         & 0.737 & 0.362 & 0.457 & 0.107 & 0.248 & 0.022 \\
\textbf{SCLCS}        & 0.127 & 0.146 & 0.096 & 0.039 & 0.068 & 0.011 \\
\textbf{SCLCS}$_{\text{Dense}}$ & 0.195 & 0.146 & 0.112 & 0.054 & 0.071 & 0.017 \\
\bottomrule
\end{tabular}
}
\end{table*}

\subsection{Extension on Data Properties}
We further explore the relationship between 'site ID' and the core-sets selected by all the baseline methods and the proposed \textbf{SCLCS} (which is blind to site). Performances are quantified by calculating the L1 Distance between the site distribution of the original dataset and the site distribution of the selected core-set. The samples of each subject are same on the original dataset but vary on the core-set. Thus we provide two levels of result based on sample and subject, respectively. Results are reported in Table~\ref{tab:site_balance}.

\textbf{SCLCS} demonstrates robustness to site imbalance. For example, at a $30\%$ sampling ratio, \textbf{SCLCS} achieves an L1 distance of $\approx0.096$, ranking second only to the variance-based method EVA ($\approx0.047$).
Crucially, \textbf{SCLCS} significantly outperforms other baselines like Entropy (0.149), which tend to be more biased towards specific sites.
This is an interesting finding. It necessitates exploring methods to balance different data properties and the performance on the proposed ranking preservation task, which is beyond the scope of our current paper.

\subsection{Robustness to Window Size}

\begin{table*}[t]
\centering
\caption{Robustness of different methods to window size}
\label{tab:horizontal_robustness}
\scriptsize 
\setlength{\tabcolsep}{2pt} 

\resizebox{\textwidth}{!}{
\begin{tabular}{l|ccccccccc|ccccccccc}
\toprule
\textbf{Method} & \multicolumn{9}{c|}{\textbf{Robustness on MDD Diagnosis Task}} & \multicolumn{9}{c}{\textbf{Robustness on Brain Fingerprinting Task}} \\
 & \multicolumn{3}{c}{nDCG@5} & \multicolumn{3}{c}{nDCG@10} & \multicolumn{3}{c|}{nDCG@20} & \multicolumn{3}{c}{nDCG@5} & \multicolumn{3}{c}{nDCG@10} & \multicolumn{3}{c}{nDCG@20} \\
\cmidrule(lr){2-4} \cmidrule(lr){5-7} \cmidrule(lr){8-10} \cmidrule(lr){11-13} \cmidrule(lr){14-16} \cmidrule(lr){17-19}
Ratio & 0.1 & 0.3 & 0.5 & 0.1 & 0.3 & 0.5 & 0.1 & 0.3 & 0.5 & 0.1 & 0.3 & 0.5 & 0.1 & 0.3 & 0.5 & 0.1 & 0.3 & 0.5 \\
\midrule
Random       & 0.213 & 0.498 & 0.802 & 0.413 & 0.611 & 0.843 & 0.558 & 0.707 & 0.886 & 0.206 & 0.285 & 0.282 & 0.234 & 0.386 & 0.353 & 0.270 & 0.443 & 0.376 \\
Forgetting   & 0.234 & 0.190 & 0.637 & 0.336 & 0.287 & 0.689 & 0.469 & 0.376 & 0.745 & 0.252 & 0.478 & 0.417 & 0.287 & 0.525 & 0.461 & 0.324 & 0.556 & 0.500 \\
Entropy      & 0.258 & 0.467 & 0.566 & 0.293 & 0.506 & 0.633 & 0.360 & 0.533 & 0.717 & 0.227 & 0.474 & 0.576 & 0.243 & 0.526 & 0.582 & 0.347 & 0.544 & 0.591 \\
EI2N         & 0.174 & 0.007 & 0.531 & 0.253 & 0.199 & 0.567 & 0.344 & 0.250 & 0.597 & 0.372 & \textbf{0.670} & 0.333 & 0.366 & \textbf{0.638} & 0.395 & 0.374 & \textbf{0.631} & 0.405 \\
AUM          & 0.452 & 0.631 & 0.701 & 0.490 & 0.725 & 0.784 & 0.528 & 0.797 & 0.852 & 0.291 & 0.690 & 0.399 & 0.324 & 0.684 & 0.483 & 0.332 & 0.728 & 0.522 \\
CCS          & 0.464 & 0.596 & 0.727 & 0.554 & 0.682 & 0.794 & 0.625 & 0.757 & 0.853 & 0.240 & 0.502 & 0.473 & 0.282 & 0.560 & 0.506 & 0.306 & 0.566 & 0.523 \\
EVA          & 0.347 & 0.422 & 0.513 & 0.456 & 0.496 & 0.597 & 0.560 & 0.580 & 0.663 & 0.430 & 0.456 & 0.444 & 0.417 & 0.510 & 0.476 & 0.474 & 0.517 & 0.488 \\
BOSS         & 0.281 & 0.369 & 0.833 & 0.380 & 0.430 & 0.867 & 0.468 & 0.493 & 0.899 & 0.062 & 0.364 & 0.669 & 0.110 & 0.379 & 0.747 & 0.151 & 0.416 & 0.752 \\
\textbf{SCLCS} & \textbf{0.621} & \textbf{0.834} & 0.708 & \textbf{0.662} & \textbf{0.818} & 0.766 & \textbf{0.735} & \textbf{0.849} & 0.826 & \textbf{0.475} & 0.569 & 0.673 & \textbf{0.491} & 0.569 & \textbf{0.750} & 0.495 & 0.593 & \textbf{0.762} \\
\textbf{SCLCS}$_\text{Dense}$ & \textbf{0.661} & 0.732 & \textbf{0.836} & \textbf{0.718} & 0.781 & \textbf{0.885} & \textbf{0.778} & 0.826 & \textbf{0.917} & 0.420 & 0.593 & \textbf{0.712} & 0.459 & 0.617 & 0.707 & \textbf{0.496} & 0.657 & 0.705 \\
\bottomrule
\end{tabular}
}
\end{table*}

We re-process the entire dataset using a different sliding window configuration (Window Size = $50$ TRs, Stride = $45$ TRs), which differs significantly from the setting used in the main paper (Window Size = $70$, Stride = $35$). We then re-calculated all 130 SPIs and re-evaluated the ranking consistency of \textbf{SCLCS} against all baselines on both downstream tasks.

The results, detailed in Table~\ref{tab:horizontal_robustness}, empirically demonstrate that \textbf{SCLCS} maintains its superior performance and ranking stability regardless of the window size, confirming that our method captures intrinsic structural patterns rather than artifacts of specific preprocessing parameters.

On MDD diagnosis task, even with shorter window lengths (which can introduce more noise), \textbf{SCLCS} and \textbf{SCLCS}$_{\text{Dense}}$ consistently outperform all baselines. Notably, at the challenging 0.1 sampling ratio, \textbf{SCLCS}$_{\text{Dense}}$ achieves an nDCG@5 of 0.661, which is significantly higher than the strongest baselines like AUM (0.452) and CCS (0.464).
This confirms that our density-aware selection strategy is highly robust to variations in temporal segmentation.

On brain fingerprinting task, \textbf{SCLCS} continues to show state-of-the-art performance, particularly at the 0.1 ratio with an nDCG@5 of $0.475$, far exceeding Random (0.206).
While baselines like El2N show high variance (performing well at $0.3$ but dropping significantly at $0.5$), \textbf{SCLCS} and \textbf{SCLCS}$_{\text{Dense}}$ exhibit a more stable performance trajectory as the sampling ratio increases (e.g., \textbf{SCLCS} improves from $0.475$ to $0.673$).

These empirical results directly demonstrate that \textbf{SCLCS} is not sensitive to the choice of window size and can reliably select high-quality core-sets across different pipeline configurations. We believe this evidence strongly supports the practical reliability of \textbf{SCLCS} in diverse experimental settings.

\section{Reproducibility Information}\label{app:repro}
\subsection{Data Preprocessing}\label{app:dp}
We use a subset of 904 subjects (458 MDD, 446 controls) from the REST-meta-MDD consortium, selected via a two-stage filtering process. First, we adopt the quality-controlled cohort defined in~\citet{yan2019reduced}, which retains $1,642$ subjects across 17 sites. Second, we intersect this cohort with the 9-site subset used in~\citet{long2020altered}, chosen for consistent acquisition parameters (3T scanners, 240 time points, TR = 2000 ms). This yields a final sample from 7 sites with harmonized imaging protocols.

The analysis focuses on $33$ regions of interest (ROIs) associated with the default mode network (DMN), as defined in the Dosenbach-160 atlas~\citep{dosenbach2010prediction}. ROI-level time series were obtained directly from the preprocessed data released by REST-meta-MDD. We use the version with global signal regression (GSR) applied, consistent with prior findings~\citep{yan2019reduced}.

Our decision to focus on the DMN-33 ROI was a principled decision based on scientific control, reproducibility, and computational feasibility:
\begin{itemize}[nosep]
    \item \textbf{Scientific Control:} To validate the proposed new 'ranking preservation' task, our first priority was to isolate variables. Our primary goal is to define the "Ranking Preservation" problem and prove that a "structure-based" approach is a feasible path. Introducing multiple atlases or pipelines would change our scenario from Benchmark (SPIs) to a different, combinatorially explosive scenario of Benchmark (SPIs x atlases/pipelines), which is a valuable, promising but separate scientific task.
    \item \textbf{Reproducibility:} This scientific control principle is reinforced by the dataset itself. To ensure the highest standard of reproducibility, we used the official, standardized pre-processed data from the REST-meta-MDD consortium. This setting ensures the results are not affected by personally defined preprocessing pipelines.
    \item \textbf{The Computational Infeasibility:} Finally, even if we were to use an official atlas like AAL-90 (N = 90), an average complexity of each SPI would increase compute time by $\approx9$x (from N = 33).
\end{itemize}

Each subject is represented by a multivariate time series of $T \times R$, where $T$ is the number of time points and $R = 33$. To standardize downstream sampling, we truncate all time series to 210 time points. We apply a sliding window of length 70 TRs with a step of 35 TRs, yielding five overlapping temporal segments per subject. This configuration is consistent with prior dynamic connectivity studies~\citep{allen2014tracking, preti2017dynamic, long2020altered}, and provides a balance between temporal resolution and estimation reliability. These segments serve as dynamic samples for our core-set selection framework, alongside static networks built from each subject’s full time series.

\begin{table}[h]
\centering
\caption{Acquisition details of the 7 selected REST-meta-MDD sites. All data were collected using 3T scanners and have consistent TR ($2,000$ ms). Minor variation exists in number of time points.}
\label{tab:site-info}

\resizebox{\linewidth}{!}{
\begin{tabular}{c|p{6.5cm}|c|c|c|c|c|c}
\midrule
\textbf{Site ID} & \textbf{Institution} & \textbf{MDD} & \textbf{HC} & \textbf{Scanner} & \textbf{TR (ms)} & \textbf{TE (ms)} & \textbf{Timepoints} \\
\midrule
15 & Zhongda Hospital, Southeast University & 37 & 30 & Siemens Verio 3T & 2000 & 25.0 & 240 \\
17 & First Affiliated Hospital of Chongqing Medical University & 41 & 41 & GE Signa 3T & 2000 & 40.0 & 240 \\
19 & Anhui Medical University & 31 & 18 & GE Signa 3T & 2000 & 22.5 & 240 \\
20 & Southwest University & 229 & 250 & Siemens Tim Trio 3T & 2000 & 30.0 & 242 \\
21 & Beijing Anding Hospital, Capital Medical University & 65 & 79 & Siemens Tim Trio 3T & 2000 & 30.0 & 240 \\
22 & Second Xiangya Hospital, Central South University & 20 & 18 & Philips Gyroscan Achieva 3.0T & 2000 & 30.0 & 250 \\
23 & West China Hospital, Sichuan University & 23 & 22 & Philips Achieva 3.0T TX & 2000 & 30.0 & 240 \\
\midrule
\end{tabular}
}
\end{table}

Across the selected sites in the REST-meta-MDD consortium, the number of retained time points after preprocessing slightly varies due to site-specific acquisition protocols. While most sites provided 230 time points, others contributed data with 232 or 240 time points. These variations are a result of both initial protocol settings and preprocessing procedures (e.g., discarding initial scans to ensure magnetization equilibrium). A summary of time point lengths by site is given in Table~\ref{tab:ts-lengths}.

\begin{table}[h]
\centering
\caption{Post-preprocessing time point lengths across selected sites.}
\label{tab:ts-lengths}
\begin{tabular}{c|p{6.5cm}|c}
\midrule
\textbf{Site ID} & \textbf{Institution} & \textbf{Timepoints (after preprocessing)} \\
\midrule
15 & Zhongda Hospital, Southeast University & 230 \\
17 & First Affiliated Hospital of Chongqing Medical University & 230 \\
19 & Anhui Medical University & 230 \\
20 & Southwest University & 232 \\
21 & Beijing Anding Hospital, Capital Medical University & 230 \\
22 & Second Xiangya Hospital, Central South University & 240 \\
23 & West China Hospital, Sichuan University & 230 \\
\midrule
\end{tabular}
\end{table}

To ensure consistency in downstream dynamic sampling, all time series were uniformly truncated to the first 210 time points. This guarantees a consistent sampling space across all subjects regardless of site-specific acquisition length.

We then applied a sliding window with a fixed length of 70 TRs and a step size of 35 TRs. This configuration yields exactly five overlapping segments per subject, each representing a snapshot of short-term functional dynamics. These segments serve as candidate samples for evaluating structural stability and selecting representative subjects in our core-set selection framework.

We used 33 regions of interest (ROIs) associated with the default mode network (DMN), selected from the Dosenbach-160 atlas~\citep{dosenbach2010prediction} following the specification provided by~\citet{yan2019reduced}. These ROIs were identified using public scripts available at \url{https://github.com/Chaogan-Yan/PaperScripts/tree/master/Yan_2019_PNAS/Dos160} and used consistently in REST-meta-MDD-related studies.

\begin{table}[h]
\centering
\caption{List of 33 DMN-related ROIs selected from the Dosenbach-160 atlas.}
\label{tab:dmn-rois}
\begin{tabular}{c|c|c}
\midrule
\textbf{ROI Index} & \textbf{ROI Type} & \textbf{Yeo Network (Label)} \\
\midrule
1 & vmPFC & 7 (DMN) \\
4 & mPFC & 7 (DMN) \\
5 & aPFC & 7 (DMN) \\
6 & vmPFC & 7 (DMN) \\
7 & vmPFC & 7 (DMN) \\
11 & vmPFC & 7 (DMN) \\
13 & vmPFC & 7 (DMN) \\
14 & ACC & 7 (DMN) \\
15 & vlPFC & 7 (DMN) \\
17 & sup frontal & 7 (DMN) \\
20 & sup frontal & 7 (DMN) \\
25 & vFC & 7 (DMN) \\
63 & inf temporal & 7 (DMN) \\
72 & inf temporal & 7 (DMN) \\
73 & post cingulate & 7 (DMN) \\
85 & precuneus & 7 (DMN) \\
90 & post cingulate & 7 (DMN) \\
91 & inf temporal & 7 (DMN) \\
93 & post cingulate & 7 (DMN) \\
94 & precuneus & 7 (DMN) \\
100 & sup temporal & 7 (DMN) \\
102 & angular gyrus & 7 (DMN) \\
104 & IPL & 7 (DMN) \\
105 & precuneus & 7 (DMN) \\
108 & post cingulate & 7 (DMN) \\
111 & post cingulate & 7 (DMN) \\
112 & precuneus & 7 (DMN) \\
115 & post cingulate & 7 (DMN) \\
117 & angular gyrus & 7 (DMN) \\
124 & angular gyrus & 7 (DMN) \\
132 & precuneus & 7 (DMN) \\
134 & IPS & 7 (DMN) \\
137 & occipital & 7 (DMN) \\
\midrule
\end{tabular}
\end{table}

\subsection{Baseline Introduction and Parameter Settings}\label{app:bias}
\paragraph{Baseline Introduction}
\begin{itemize}

  \item \textbf{Random}: Uniformly samples instances without considering model behavior or data statistics.

  \item \textbf{k-Means}~\citep{hartigan1979algorithm}: An unsupervised clustering algorithm that partitions data into k clusters based on feature similarity, with the core-set formed by selecting samples closest to the cluster centroids. 

  \item \textbf{Forgetting}~\citep{toneva2018empirical}: Ranks samples by the number of times they transition from correct to incorrect predictions during training.

  \item \textbf{Entropy}~\citep{Coleman2020Selection}: Scores samples using the entropy of model output probabilities to reflect prediction uncertainty.

  \item \textbf{Area Under the Margin (AUM)}~\citep{pleiss2020identifying}: Computes the average margin between the true class probability and the highest non-true probability across epochs.

  \item \textbf{Example-Level L2 Norm (EL2N)}~\citep{paul2021deep}: Measures the L2 norm between model predictions and true labels over early training epochs as a proxy for example difficulty.

  \item \textbf{Coverage-Centric Selection (CCS)}~\citep{zheng2022coverage}: Stratifies samples by importance scores (e.g., AUM) and performs balanced sampling across strata to preserve distributional coverage.

  \item \textbf{Evolution-aware Variance (EVA)}~\citep{hong2024evolution}: Aggregates prediction error variances within early and late training windows to capture evolving sample dynamics.

  \item \textbf{Balanced One-shot Subset Selection (BOSS)}~\citep{acharya2024balancing}: Greedily selects a subset by maximizing a Beta-weighted objective over feature similarity, label variability, and difficulty-based scores.

\end{itemize}
\paragraph{Parameter Settings}
All baseline methods are evaluated using a unified training setup. The classification model is a compact residual network designed for multivariate time series inputs. It consists of a stem convolution followed by three residual blocks with output channels of 32, 64, and 128, respectively. A global average pooling layer and a fully connected classifier complete the architecture. This design balances expressiveness and efficiency for medium-scale time series classification.

The model is trained for 200 epochs using the Adam optimizer with a learning rate of 0.01, weight decay of 1e-4, and a batch size of 256. All methods, including full-data training and subset-based training, use identical configurations.

Per-sample importance scores are derived from training dynamics and used to evaluate a range of sampling strategies. For EVA, we compute the variance of the prediction error vector within two non-overlapping windows: epochs 100–109 and 190–199, following the original protocol.

Hyper parameters of \textbf{SCLCS} are optimized using grid searching in the flowing spaces:
\begin{enumerate}
  \item Dimension of Transformer: [4, 8, 16, 32, 64, 128, 256].
  \item Head Number: [2, 4, 8, 16, 32, 64].
  \item Epochs for calculating \textbf{SPS}: [50, 100, 150, 200].
\end{enumerate}

\subsection{Environment}
Experiments are performed on an 8-GPU (H20) high-performance computing cluster provided by the Large-scale Instrument Sharing Platform of Southwest University.

\section{Limitations}\label{Limi}
While this work establishes a new paradigm for efficient FC benchmarking, we identify several exciting avenues for future investigation that build upon our findings:

\begin{itemize}
    \item \textbf{Deepening the SPS-SPI Connection:} Our results demonstrate a strong empirical link between low \textbf{SPS} (structural stability) and effective core-set selection. However, the precise theoretical mechanism connecting the training dynamics of our encoder to the performance ranking of diverse, external SPI models warrants deeper investigation. Future work could explore this link to build a more formal bridge between learnable latent structures and statistical model behavior.

    \item \textbf{Developing an Adaptive Sampling Strategy:} Our experiments show that the optimal choice between simple low-SPS ranking (\textbf{SCLCS}) and density-aware sampling (\textbf{SCLCS\_Dense}) is task-dependent. A key next step is to develop heuristics or a meta-learning framework to automatically select the optimal sampling strategy based on dataset characteristics (e.g., class structure, sample heterogeneity), removing the need for manual selection.

    \item \textbf{Broader Generalization and Application:} While our evaluation on the large-scale, heterogeneous REST-meta-MDD dataset provides a strong test of robustness, the framework's generalizability should be further validated. A full study of supervision (site-ID, multi-task, and more) and how SPI properties modulate task difficulty is a valuable future direction. The relationship between the RP problem and data properties is another major extension to our work. Future studies should also apply \textbf{SCLCS} to fMRI datasets from different clinical populations (e.g., Alzheimer's disease, ADHD), other imaging modalities, or even entirely different domains of multivariate time-series analysis to establish the full scope of its applicability.
    \item \textbf{More Heterogeneous Scenarios:} Simply pulling resting-state and task-fMRI scans together into the contrastive loss would be problematic, as the functional state changes. This poses a more complex heterogeneous data fusion challenge. Exploring model benchmarking problem on heterogeneous scan types is a valuable future direction.
\end{itemize}

\section*{The Use of Large Language Models (LLMs)}
In the preparation of this manuscript, we utilize Large Language Models (LLMs) in a supportive capacity. Specifically, their use is confined to the following areas:
\begin{itemize}
    \item \textbf{Writing and Editing:} LLMs are employed to enhance the clarity, grammar, and style of the text, ensuring the manuscript's readability.

    \item \textbf{Assistance with Theorem Proofs:} LLMs serves as an assistive tool to verify the logical consistency and correctness of individual steps within our mathematical derivations and proofs. 
\end{itemize}
The core scientific ideas, the structure of the proofs, the experimental design, and all final conclusions presented in this paper are conceived and developed entirely by the authors.
\end{document}